\documentclass[accepted]{uai2025} 
                        
\usepackage[american]{babel}

\usepackage{natbib} 
\usepackage{mathtools} 
\usepackage{booktabs} 
\usepackage{tikz} 
\tikzset{
  node style/.style={
    circle, 
    draw=black, 
    line width=1pt, 
    minimum size=15pt, 
    align=center,
    inner sep=0pt,
    text width=19pt
  },
  directed edge/.style={-Latex, thick},
  bidirected edge/.style={Latex-Latex, line width=1.05pt, dashed}
}
\usepackage{mathtools}
\usepackage{amssymb}
\usepackage{amsthm}
\usepackage{amsmath}
\usepackage{amsfonts}
\usepackage{thm-restate}

\usepackage{microtype}
\usepackage{graphicx}
\usepackage{subcaption}
\usepackage{booktabs} 
\usepackage{arydshln}
\usepackage{algorithm}
\usepackage[noend]{algorithmic}
\usepackage{forloop}

\newcommand{\Pa}[2]{\textit{Pa}_{#2}(#1)}

\newcommand{\V}[0]{\mathbf{V}}

\newcommand{\E}[0]{\mathbb{E}}
\newcommand{\N}[0]{\mathbb{N}}

\newcommand{\G}[0]{\mathcal{G}}
\newcommand{\M}[0]{\mathcal{M}}
\newcommand{\F}[0]{\mathcal{F}}

\newcommand{\eps}{\epsilon}

\newtheorem{assumption}{Assumption}
\newtheorem{definition}{Definition}
\newtheorem{remark}{Remark}
\newtheorem{proposition}{Proposition}

\title{ Causal Effect Identification in Heterogeneous Environments from Higher-Order Moments}

\author[1]{\href{mailto:<yaroslav.kivva@epfl.ch>?Subject=Your UAI 2025 paper}{Yaroslav Kivva}{}}
\author[1]{Sina Akbari}
\author[2]{Saber Salehkaleybar}
\author[1]{Negar Kiyavash}
\affil[1]{%
    Ecole Polytechnique Fédérale de Lausanne, Lausanne, Switzerland
}
\affil[2]{%
    Leiden Institute of Advanced Computer Science, Leiden University, Netherlands
}
  
\begin{document}
\maketitle

\begin{abstract}
We investigate the estimation of the causal effect of a treatment variable on an outcome in the presence of a latent confounder. 
We first show that the causal effect is identifiable under certain conditions when data is available from multiple environments, provided that the target causal effect remains invariant across these environments.
Secondly, we propose a moment-based algorithm for estimating the causal effect as long as only a single parameter of the data-generating mechanism varies across environments --
whether it be the exogenous noise distribution or the causal relationship between two variables.
Conversely, we prove that identifiability is lost if both exogenous noise distributions of both the latent and treatment variables vary across environments.
Finally, we propose a procedure to identify which parameter of the data-generating mechanism has varied across the environments and evaluate the performance of our proposed methods through experiments on synthetic data.

\end{abstract}

\section{Introduction}\label{sec:intro}

Identifying the causal effect of a treatment on an outcome is a fundamental objective in various fields, including economics \citep{card1993minimum,angrist1991does}, social sciences \citep{rosenbaum1983central,imbens2024causal}, epidemiology \citep{robins2000marginal}, and artificial intelligence \citep{pearl2009causality,pearl2014probabilistic}. 
One of the primary challenges in causal effect identification is the presence of latent confounders -- unobserved variables that influence both the treatment and the outcome. Ignoring latent confounders and simply regressing the outcome on the treatment can lead to biased estimates of the causal effect. To address the challenge of latent confounding, one might conduct a randomized controlled trial (RCT). However, RCTs are often too expensive, time-consuming, or even infeasible due to ethical or legal constraints.

In many real-life applications, data collected from different domains often exhibit heterogeneity due to variations in the causal mechanisms that generate each variable from its direct causes. There is extensive research in the literature (see Section \ref{sec:related_work} for the related work) on causality that leverages data from multiple environments to recover causal relationships. In particular, several studies \citep{ghassami2017learning,huang2020causal,jaber2020causal} have shown that data collected from multiple environments can narrow the set of possible causal graphs compatible with the observed data, compared to using data from a single environment.

Research in multi-environment settings follows two main approaches (see Section \ref{sec:related_work} for more details). The first aims to identify an equivalence class of causal structures by leveraging distributional shifts across environments, assuming these arise from unknown interventions. The second approach focuses on identifying the direct causes of a target variable rather than the entire causal graph, often assuming linear causal mechanisms and estimating causal effects corresponding to the direct causes. 

In this paper, we study the problem of identifying the causal effect of a treatment $T$ on an outcome $Y$ within linear structural causal models (SCMs) in a multi-environment setting (see Figure \ref{fig: main graph}), where a latent confounder $U$ between $T$ and $Y$ is present. This problem closely relates to the second approach discussed earlier. However, prior work either does not account for latent confounding between treatment and outcome or imposes restrictive assumptions on which causal mechanisms can vary across environments (see Section \ref{sec:related_work} for more details).

The main contributions of our work are as follows:
\begin{itemize}
    \item For the setup considered in Figure \ref{fig: main graph} with two environments, we show (Theorem \ref{th:unknown factor disparity}) that if there is only a single \textit{unknown} change across two domains, 
    then we can classify whether this change occurs in the causal effect of the latent confounder $U$ on treatment $T$ or outcome $Y$ (i.e. only $\alpha$ or $\gamma$ varies between the two environments), or in the exogenous noises of variables $T$ or $U$. 
    Furthermore, in the case that the varying parameter is either the causal effect of $U$ on $T$ or that of $U$ on $Y$, then the causal effect of the treatment $T$ on $Y$ is identifiable uniquely, and otherwise it can be recovered up to two possible candidates.
    We provide an estimation procedure tailored to each of these cases.

    \item We provide a non-identifiability result (Theorem \ref{th:non-identifiable beta}) showing that the causal effect is not identifiable from two environments if both the exogenous noises of $T$ and $U$ vary across the environments.
    \item We provide extensive experimental results validating our algorithms.
    These results show that our proposed estimators consistently converge to the true value of the treatment effect, whereas linear regression baselines exhibit systematic bias.
    In addition, we analyze the typical range of key parameters in our algorithms. 
    Our code is \href{https://github.com/SinaAkbarii/IdentificationMultipleDomain}{provided online} for reproducibility purposes.
\end{itemize}

\section{Preliminaries}

\subsection{Notation}
Let $G = \langle \V, \mathbf{E} \rangle$ be an \textit{directed acyclic graph} (DAG), such that each vertex represents a random variable and each edge corresponds to the direct causal relationship between the random variables it connects.  A DAG $\G$ with causal relations defines a structural causal model $\M$ (SCM) such that any random variable $X$ in the graph satisfies the structural equation
\begin{equation*}
    X = f_{x}(\Pa{X}{}, \eps_{x}),
\end{equation*}
where $\Pa{X}{}$ denotes parents of $X$ in DAG $\G$, $\eps_x$ is the exogenous noise corresponding to $X$, and  $f_x(\cdot)$ is a function capturing how variable $X$ causally depends on its parents in the causal graph. The subscript in each exogenous noise denotes the variable to which it corresponds, e.g., $\eps_x$ is the noise pertaining to $X$. 
All exogenous noises in a structural equation model are assumed to be jointly independent. 
To indicate that the structural causal model $\M$ corresponds to an environment $i$,  $(i)$ is added as a superscript, i.e. $\M^{(i)}$. 
Similarly, $X^{(i)}$, $\eps^{(i)}$, $f_{x}^{(i)}$,  denote a random variable, an exogenous noise, and a causal relationship in the environment $i$. For ease of presentation, we may omit the superscript $(i)$ if the index of the environment is not important or is clear from the context.

For any variable $X\in \V$, the intervention $do(X=x)$ is an operation that converts SCM $\M$ to a new one where the equation of $X$ in $\M$ is replaced by the constant $x$. Intuitively, this operation can be seen as performing an experiment where one forces a variable $X$ to take a specific value $x$.

\subsection{Identifiability from multi-domain observations}
Let $\V$ be the set of random variables in SCM $\M$ and let observed random variables $T, Y \in \V$ denote the treatment and the outcome variable, respectively. 
In this paper, we focus on linear SCMs, a widely adopted assumption in the literature.  
\begin{assumption}[Linear SCM]
    \label{ass:linear SCM}
    For any random variable $X\in \V$, function $f_{x}(\cdot)$ is a linear function: 
    \begin{equation*}
        X = \sum_{S\in \Pa{X}{}} \alpha_{S, X} S + \eps_x.
    \end{equation*}
\end{assumption}
Note that coefficient $\alpha_{S, X}$ represents the direct causal effect of the variable $S$ on $X$.
The causal effect of the treatment $T$ on the outcome $Y$ is defined as $\E[Y|do(T=1)] - \E[Y|do(T=0)]$. Under Assumption \ref{ass:linear SCM} (Linear SCM), finding this causal effect is equivalent to learning coefficient $\beta:=\alpha_{Y, T}$. 

In this paper, we assume the observational data comes from a collection of linear SCMs $\M^{(1)}, \M^{(2)} \dots, \M^{(n)}$ that satisfy the following assumption. 
\begin{assumption}
    \label{ass:invariant treatment effect}
    The causal effect of the treatment $T$ on the outcome $Y$ is invariant across domains, that is $\beta^{(1)}=\beta^{(2)}=\dots=\beta^{(n)}$.
\end{assumption}
Assumption  \ref{ass:invariant treatment effect} states that the treatment effect remains the same across the domains. For instance,  in example proposed by \cite{shi2021invariant}, one may be interested in the effect of sleeping pills on lung disease using electronic health records collected from multiple hospitals. The causal effect of sleeping pills on lung disease is assumed to remain consistent across different hospitals.

We denote by $\F\left( \M^{(i)}, \M^{(j)}\right)$, the set of the \textbf{non-invariant} coefficients and exogenous noises between two SCMs $\M^{(i)}$ and $\M^{(j)}$. For example, consider two equations $X^{(i)} = \alpha^{(i)}\eps_u^{(i)} + \gamma^{(i)}\eps_{d}^{(i)} + \eps_{x}^{(i)}$ and \mbox{$X^{(j)} = \alpha^{(j)}\eps_u^{(j)} + \gamma^{(j)}\eps_{d}^{(j)} + \eps_x^{(j)}$}, where $\alpha^{(i)} = \alpha^{(j)}$, $\gamma^{(i)} \neq \gamma^{(j)}$, $\eps_u^{(i)} \not \sim \eps_{u}^{(j)}$, $\eps_x^{(i)} \not \sim \eps_{x}^{(j)}$ and $\eps_d^{(i)} \sim \eps_{d}^{(j)}$ (notation $\sim$ means the random variables are drawn from the same distribution).  Then  $\alpha, \eps_d \not \in \F\left( \M^{(i)}, \M^{(j)}\right)$ while \mbox{$\gamma, \eps_u, \eps_x \in \F\left( \M^{(i)}, \M^{(j)}\right)$}.

Additionally, we require all exogenous noises to have finite moments and be ``well-defined'' given the moments.
\begin{assumption}[Finite moments]
    \label{ass:finite moments}
    Given an SCM $\M$ on the set of variables $\V$, for any $X \in \V$ and for any $n\in \N$,   $\E\left[ \eps_x^n \right] < \infty$.
\end{assumption}

\begin{assumption}
    \label{ass:distribution via moments}
    Given an SCM $\M$ on the set of variables $\V$, for any $X \in \V$, there exists some $s>0$ such that the power series $\sum_k\E\left[ \eps_x^{k}\right]r^k / k!$ converges for any $0<r<s$.
\end{assumption}
The last assumption implies that the distribution of the random variable $X$ is \textit{uniquely} determined given its moments.

Next,  we formalize the definition of identifiability of the treatment effect.
\begin{definition} \textbf{(Identifiability)}
    Suppose Assumptions \ref{ass:linear SCM} through \ref{ass:distribution via moments} hold. Moreover, assume that there are $n$ environments, each with a true underlying SCM denoted by $\mathcal{M}^{(i)}$ for environment $i$. The treatment effect $\beta$ is said to be \emph{identifiable} from merely observational distributions of $n$ environments if for any collection of SCMs $\{\tilde{\mathcal{M}}^{(i)}\}_{i=1}^n$ such that $\tilde{\mathcal{M}}^{(i)}$ entails the same observational distribution as $\mathcal{M}^{(i)}$ for every $i\in \{1,\cdots,n\}$,
    then the treatment effect in the collection $\{\tilde{\mathcal{M}}^{(i)}\}_{i=1}^n$ is equal to the one in $\{\mathcal{M}^{(i)}\}_{i=1}^n$; that is, $ \tilde{\beta}=\beta$.
\end{definition}

In this work, we address the problem for the canonical case of two environments (i.e., $n=2$). This characterization suffices as the identifiability results extend  to larger values of $n$ simply by considering the environments in pairs.

\section{Main result}
We consider the problem of estimating the causal effect of treatment $T$ on the outcome $Y$ in 
DAG $\G$ given in Figure~\ref{fig: main graph}. It is well-known that given observational data from a single environment, this causal effect is not identifiable \citep{salehkaleybar2020learning}. We will show that the causal effect can be identified given observational data from two environments under certain mild assumptions.
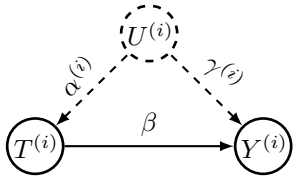
\begin{figure}[h!]
  \begin{center}
    \begin{tikzpicture}[
            roundnode/.style={circle, draw=black!60,, fill=white, thick, inner sep=1pt},
            dashednode/.style = {circle, draw=black!60, dashed, fill=white, thick, inner sep=1pt},
            ]
            \node[node style]        (T)        at (-1.5, 0)                   {$T^{(i)}$};
            \node[node style]        (Y)        at (1.5, 0)                    {$Y^{(i)}$};
            \node[node style, dashed]       (U)        at (0, 1.5)                    {$U^{(i)}$};
            
            \draw[-latex, thick] (T.east) -- (Y.west) node[midway,sloped,above] {$\beta$} ;
            \draw[latex-, dashed, thick] (T) -- (U) node[midway,sloped,above] {$\alpha^{(i)}$};
            \draw[latex-, dashed, thick] (Y) -- (U) node[midway,sloped,above] {$\gamma^{(i)}$};
        \end{tikzpicture}
  \end{center}
  \caption{Causal graph of a linear SCM in the $i$-th domain.}
  \label{fig: main graph}
\end{figure}

More specifically, in each domain $i\in\{1,2\}$, we consider the following linear SCM (with the corresponding causal graph in Figure \ref{fig: main graph}) with a treatment variable $T^{(i)}$, an outcome variable $Y^{(i)}$, and a latent confounder $U^{(i)}$ :
\begin{equation}
\begin{cases}
        & U^{(i)} := \epsilon^{(i)}_u, \\
        & T^{(i)} := \alpha^{(i)}U^{(i)} +\epsilon^{(i)}_{t}, \\
        & Y^{(i)} := \beta T^{(i)} + \gamma^{(i)} U^{(i)} + \epsilon^{(i)}_y,
\end{cases}  
\label{eq:main_SEM}
\end{equation}
where $\epsilon^{(i)}_u$, $\epsilon^{(i)}_t$, $\epsilon^{(i)}_y$ are exogenous noises corresponding to $U^{(i)}$, $T^{(i)}$, and $Y^{(i)}$, respectively. 
In the sequel, we use $\beta$ and `the treatment effect' interchangeably.

In this section, we will show that if $\F\big( \M^{(1)}, \M^{(2)}\big)$ is known and $|\F\big( \M^{(1)}, \M^{(2)}\big)| = 1$, and additionally $\eps_y \not \in \F\big( \M^{(1)}, \M^{(2)}\big)$, then the treatment effect can be uniquely identified under mild non-Gaussianity assumptions. We will propose a procedure to learn $\beta$ for any given 
$\F\big( \M^{(1)}, \M^{(2)}\big)$ satisfying the aforementioned conditions. 
Note that the case $\F\big( \M^{(1)}, \M^{(2)}\big) = \{\eps_y\}$ is not of interest, since intuitively, the change in the distribution of $\eps_y$ does not provide any new information on the treatment mechanism. 
Moreover, it can be shown that the treatment effect $\beta$ is not uniquely identifiable for such a scenario. For completeness, we provide proof of this statement in Proposition \ref{prop:non-id eps_y} in Appendix \ref{apx:proofs}.

In practice, we might only know that $|\mathcal{F}(\mathcal{M}^{(1)}, \mathcal{M}^{(2)})|=1$, without knowing which parameter has changed across environments. 
The following result indicates that even in such a scenario, $\beta$ can be uniquely identified in some cases, and identified up to a finite set in the others.
\begin{restatable}{theorem}{thmunknownfactor}
    \label{th:unknown factor disparity}
    Consider two linear SCMs $\M^{(1)}, \M^{(2)}$ compatible with the graph of Figure \ref{fig: main graph}, such that $|\F\big( \M^{(1)}, \M^{(2)}\big)|=1$. 
    The treatment effect $\beta$ can be uniquely identified if $\F\big( \M^{(1)}, \M^{(2)}\big)\subset \{\alpha, \gamma\}$ under some additional case-specific mild assumptions; 
    otherwise, if $\F\big( \M^{(1)}, \M^{(2)}\big)\subset \{\eps_t, \eps_u\}$, $\beta$ can be identified only up to two possible candidates.
\end{restatable}
The proofs of all our results are given in Appendix \ref{apx:proofs}.\\
Below, we first show that as long as we \textit{know} which single parameter or exogenous noise variable has changed across domains, we can identify $\beta$ uniquely. Specifically, the procedures for identifying $\beta$ when $\F\big( \M^{(1)}, \M^{(2)}\big)=\{\eps_t\}$, $\F\big( \M^{(1)}, \M^{(2)}\big)=\{\eps_u\}$, $\F\big( \M^{(1)}, \M^{(2)}\big)=\{\gamma\}$ and $\F\big( \M^{(1)}, \M^{(2)}\big)=\{\alpha\}$
are given in Sections \ref{sec:case eps_t}, \ref{sec:case eps_u}, \ref{sec:case gamma} and \ref{sec:case alpha}, respectively.
Next, in Section \ref{sec:source_of_change}, we outline the procedure for identifying the varying parameter across the two environments given that $\vert\F\big(\M^{(1)},\M^{(2)}\big)\vert=1$. In particular, we can always distinguish whether the source of change was $\gamma$ or $\alpha$ across the domains. However, if $\F\big( \M^{(1)}, \M^{(2)}\big)\subset \{\eps_u, \eps_t\}$, we cannot pinpoint whether the change was due to variation in the distribution of $\eps_u$, or $\eps_t$. Therefore, we need to apply both procedures in Sections \ref{sec:case eps_t} and \ref{sec:case eps_u}  to recover two candidates for the treatment effect $\beta$.

In Section \ref{sec:non-ID-result}, we prove a non-identifiability result, namely, if the distributions of both exogenous noises $\epsilon_t$ and $\epsilon_u$ vary across the two environments, then $\beta$ is not identifiable.

Note that all the results presented in this work can be easily generalized to settings with observed confounders $\mathbf{X} = \{X_1, \dots, X_m\}$, e.g., Figure \ref{fig: graph extended with covariates}. More specifically, by regressing the treatment and outcome on the observed covariates and working with the residuals, the problem reduces to the case without observed covariates -- a similar procedure was done in \citep{kivva2024cross}.

\begin{figure}[h!]
  \begin{center}
    \begin{tikzpicture}[
            roundnode/.style={circle, draw=black!60,, fill=white, thick, inner sep=1pt},
            dashednode/.style = {circle, draw=black!60, dashed, fill=white, thick, inner sep=1pt},
            ]
            \node[node style]        (T)        at (-1.5, 0)                   {$T^{(i)}$};
            \node[node style]        (Y)        at (1.5, 0)                    {$Y^{(i)}$};
            \node[node style, dashed]       (U)        at (0, 1.5)                    {$U^{(i)}$};
            \node[node style]       (Xj)        at (-2, 1.5)                    {$X^{(i)}_j$};
            \node[node style]       (Xl)        at (2, 1.5)                    {$X^{(i)}_l$};
            
            \draw[-latex, thick] (Xj) -- (T) ;
            \draw[-latex, thick] (Xj) -- (Y) ;
            \draw[-latex, thick] (Xl) -- (T) ;
            \draw[-latex, thick] (Xl) -- (Y) ;
            \draw[-latex, thick] (T) -- (Y) ;
            \draw[latex-, dashed, thick] (T) -- (U) ;
            \draw[latex-, dashed, thick] (Y) -- (U);
        \end{tikzpicture}
  \end{center}
  \caption{Causal graph with observed covariates $\mathbf{X}$.}
  \label{fig: graph extended with covariates}
\end{figure}
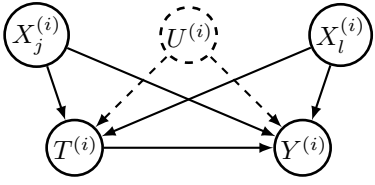

\subsection{The case \texorpdfstring{$\F\big( \M^{(1)}, \M^{(2)}\big) = \{\eps_t\}$}{F(M1,M2)=epsilon t}}
\label{sec:case eps_t}
Here we consider the case where the distribution of $\eps_t$ is changing across the two environments. 
The corresponding SCMs for environments $1$ and $2$ can be simplified to\footnote{Please note that while the distributions of $\eps_u$ and $\eps_y$ remain unchanged across the domains, the realizations of $\eps_u$ and $\eps_y$ differ between the two domains.
}
\begin{equation}
\begin{cases}
        & U^{(i)} := \epsilon_u, \\
        & T^{(i)} := \alpha U^{(i)} +\epsilon^{(i)}_{t}, \\
        & Y^{(i)} := \beta T^{(i)} + \gamma U^{(i)} + \epsilon_y.
\end{cases}  
\label{eq:case 1}
\end{equation}

\begin{restatable}{theorem}{thmcaseone}
    \label{th:case 1}
    Suppose $\M^{(1)}, \M^{(2)}$ are linear SCMs compatible with the DAG of Figure \ref{fig: main graph}, such that
    $\F\big( \M^{(1)}, \M^{(2)}\big)=\{\eps_t\}$. 
    Then under Assumptions \ref{ass:linear SCM}-\ref{ass:distribution via moments}, the treatment effect $\beta$ can be recovered uniquely.
\end{restatable}
The proof of the above result can be found in Appendix \ref{apx:proofs}. Algorithm \ref{alg:case 1} follows from the proof and outlines the procedure for estimating the treatment effect $\beta$ in this case.
\begin{algorithm}[ht]
    \caption{$\F\big( \M^{(1)}, \M^{(2)}\big)=\{\eps_t\}$}
    \label{alg:case 1}
    \textbf{Input:} $\{T^{(i)}, Y^{(i)}\}$ and $\F\big(\M^{(1)}, \M^{(2)}\big) \!=\! \{\eps_t\}$ 
    
    \begin{algorithmic}[1]
    \STATE $k \gets 1$
    \WHILE{$\E\big[\left(T^{(1)}\right)^k\big]= \E\big[\big(T^{(2)}\big)^k\big]$}
        \STATE $k \gets k+1$
    \ENDWHILE
    \STATE $\beta \gets \frac{\E\big[Y^{(1)}(T^{(1)})^{k-1}- Y^{(2)}(T^{(2)})^{k-1}\big]}{\E\left[(T^{(1)})^k - (T^{(2)})^k\right]}$
    \STATE \textbf{RETURN:} $\beta$
    \end{algorithmic}
\end{algorithm}

Here, at the end of \textbf{while} loop, we find the smallest $k$ such that 
\begin{equation*}
    \E\big[\big(\epsilon^{(1)}_{t}\big)^k\big]\neq \E\big[\big(\epsilon^{(2)}_{t}\big)^k\big], 
\end{equation*}
which is required to estimate $\beta$ with the formula given in line $4$ of Algorithm \ref{alg:case 1}. 
\begin{remark}
    Note that under Assumption \ref{ass:distribution via moments}, both distributions $\eps_t^{(1)}$ and $\eps_t^{(2)}$ are uniquely defined given all the moments. 
    Since these distributions are different, they differ at least in one of their moments, which guarantees that such a $k$ exists. 
\end{remark}

\subsection{The case \texorpdfstring{$\F\big( \M^{(1)}, \M^{(2)}\big) = \{\eps_u\}$}{F(M1,M2)=epsilon u}}
\label{sec:case eps_u}
In this section, we assume that only the distribution of $\eps_u$ is changing across the two environments.
The SCM equations can be reduced similarly to Equation \ref{eq:case 1}, as in the previous case.
However, the assumptions for the identifiability of the treatment effect $\beta$ are slightly stronger.

\begin{restatable}{theorem}{thmcasetwo}
    \label{th:case 2}
    Suppose $\M^{(1)}, \M^{(2)}$ are linear SCMs compatible with the DAG of Figure \ref{fig: main graph}, such that
    $\F\big( \M^{(1)}, \M^{(2)}\big)=\{\eps_u\}$. 
    Suppose that $\exists n\in \N$ such that $\E[\eps_t^n] \neq (n-1)\E\left[ \eps_t^{n-2}\right]\E\left[ \eps_t^{2}\right]$. Then under Assumptions \ref{ass:linear SCM}-\ref{ass:distribution via moments}, the treatment effect $\beta$ can be recovered uniquely.
\end{restatable}
In the statement of the theorem above, we have an additional restriction on the exogenous noise of the treatment, which under Assumption \ref{ass:distribution via moments}, is equivalent to $\eps_t$ not being Gaussian. 
Based on the proof of the theorem, we provide Algorithm \ref{alg:case 2} for estimating the treatment effect $\beta$ in this case.

\begin{algorithm}[ht]
    \caption{$\F\big( \M^{(1)}, \M^{(2)}\big)=\{\eps_u\}$}
    \label{alg:case 2}
    \textbf{Input:} $\{T^{(i)}, Y^{(i)}\}$ and $\F\big( \M^{(1)}, \M^{(2)}\big) = \{\eps_u\}$ 
    
    \begin{algorithmic}[1]
    \STATE $k \gets 1$
    \WHILE{$\E\big[\big(T^{(1)}\big)^k\big]= \E\big[\big(T^{(2)}\big)^k\big]$}
        \STATE $k \gets k+1$
    \ENDWHILE
    \STATE $r_1 \gets \frac{\E\left[Y^{(1)}(T^{(1)})^{k-1}- Y^{(2)}(T^{(2)})^{k-1}\right]}{\E\left[(T^{(1)})^k - (T^{(2)})^k\right]}$
    \STATE $r_2 \gets \textit{GetRatio}\left(r_1T^{(1)} - Y^{(1)}, T^{(1)}\right)$
    \STATE $\beta \gets r_1 - r_2$
    \STATE \textbf{RETURN:} $\beta$
    \end{algorithmic}
\end{algorithm}

Similarly to the previous case, Algorithm \ref{alg:case 2} identifies through the \textbf{while} loop the smallest $k$ for which
\begin{equation*}
    \E\big[\big(\epsilon^{(1)}_{u}\big)^k\big]\neq \E\big[\big(\epsilon^{(2)}_{u}\big)^k\big].
\end{equation*}
By knowing $k$ then in line 4 we compute $\beta + \frac{\gamma}{\alpha}$ and denote it by $r_1$.
The algorithm then obtains the value of $\frac{\gamma}{\alpha}$, denoted as  $r_2$, using the function $\textit{GetRatio}(\cdot)$, which was proposed by \cite{kivva2024cross}. We recover $\beta$ by subtracting $r_2$ from $r_1$.
Below we explain the workings of function $\textit{GetRatio}(\cdot)$.
Suppose we observe two random variables $X_1$ and $X_2$ that can be represented as
\begin{align*}
    & X_1 = a \eps + \eps_1, \\
    & X_2 = b \eps + \eps_2,
\end{align*}
where $\eps_1$, $\eps_2$ and $\eps$ are mutually independent. 
Then $\textit{GetRatio}(X_1, X_2)$ computes the ratio $a / b$ under the same assumption on the distribution of $\eps$ as in  Theorem \ref{th:case 2} imposes on the distribution of $\eps_t$. 
Moreover, it is important to emphasize that algorithm $\textit{GetRatio}(\cdot)$ does not require the knowledge of the constant $n$; it only requires the existence of such $n$.
For further details, see \citet{kivva2024cross}.

\subsection{The case \texorpdfstring{$\F\big( \M^{(1)}, \M^{(2)}\big) = \{\gamma\}$}{F(M1,M2)=gamma}}
\label{sec:case gamma}
Here, we assume that the causal effect of the latent confounder on the treatment varies across environments. 
For simplicity of notation, specifically for this setting, we rescale the exogenous noise $\eps_u$ together with $\alpha$ and $\gamma^{(i)}$ in the SCMs $\M^{(1)}$ and $\M^{(2)}$ so that $\eps_u \gets \alpha\eps_u$, $\alpha \gets 1$, $\gamma^{(i)} \gets \gamma^{(i)}/ \alpha$. Here we used that $\alpha$ and $\eps_u$ are invariant across environments, and therefore the corresponding SCM after rescaling will take the following form:
\begin{equation}
\begin{cases}
        & U^{(i)} := \epsilon_u, \\
        & T^{(i)} := U^{(i)} +\epsilon_{t}, \\
        & Y^{(i)} := \beta T^{(i)} + \gamma^{(i)} U^{(i)} + \epsilon_y,
\end{cases}  
\label{eq:case 3}
\end{equation}
where $\eps_u$ and $\gamma^{(i)}$ are rescaled and the distributions of all exogenous noises are the same across environments, but not there realizations.

\begin{restatable}{theorem}{thmcasethree}
    \label{th:case 3}
    Suppose $\M^{(1)}, \M^{(2)}$ are linear SCMs compatible with the DAG of Figure \ref{fig: main graph}, such that
    $\F\big( \M^{(1)}, \M^{(2)}\big)=\{\gamma\}$.
    Suppose $\exists n \in \N$ such that $\E\left[ \eps_t^n \right] \neq (n-1) \E\left[\eps_t^{n-2}\right] \E\left[ \eps_t^2\right]$.
    Then under Assumptions \ref{ass:linear SCM}-\ref{ass:distribution via moments}, the treatment effect $\beta$ can be recovered uniquely.
\end{restatable}
The proof of this theorem can be found in Appendix \ref{apx:proofs}. 
Below, we provide a procedure for estimating the treatment effect $\beta$, which consists of multiple steps. 

\textbf{Step 1.} First, we compute $ 2\beta + \gamma^{(1)} + \gamma^{(2)}$ as
\begin{equation}\label{eq:2beta}
    2\beta + \gamma^{(1)} + \gamma^{(2)} = \frac{\E\left[\left(Y^{(2)}\right)^2 - \left(Y^{(1)}\right)^2\right]}{\E\left[Y^{(2)}T^{(2)} - Y^{(1)}T^{(1)}\right]}.
\end{equation}
We then define new variables $X^{(1)}$, $X^{(2)}$ as follows
\begin{equation}
    X^{(i)} \coloneqq \left(2\beta + \gamma^{(i)} + \gamma^{(i)}\right)T^{(i)} - 2Y^{(i)}.
\end{equation}
We also define $a \coloneqq \gamma^{(2)} - \gamma^{(1)}$, and $b \coloneqq \gamma^{(2)} + \gamma^{(1)}$.

\textbf{Step 2.} We use the following equations to compute the values $\Tilde{a}\coloneqq a\E[\epsilon_u^2]$ and $\Tilde{b}\coloneqq b\E[\epsilon_t^2]$:
\begin{align*}
    & a\E\left[ \eps_u^2\right] = \frac{1}{2}\left(\E\left[T^{(1)}X^{(1)}\right] -  \E\left[T^{(2)}X^{(2)}\right]\right),\\
    & b\E\left[ \eps_t^2\right] = \frac{1}{2}\left(\E\left[T^{(1)}X^{(1)}\right] +  \E\left[T^{(2)}X^{(2)}\right]\right).
\end{align*}
Since $\E\left[ \eps_u^2\right]$ and $\E\left[ \eps_t^2\right]$ are always positive, $\Tilde{a}$ and $\Tilde{b}$ reveal the signs of $a$ and $b$.

\textbf{Step 3.} For every $n$, we define $\phi_n^{(i)}$ as follows.
If $n$ is odd, 
\begin{equation*}
    \phi_n^{(i)} = \E\big[\big(T^{(i)}\big)^{n-1}X^{(i)}\big],
\end{equation*}
and if $n$ is even,
\begin{align*}
    & \phi_n^{(1)} \!=\! \E\big[X^{(1)}\big(T^{(1)}\big)^{n-1}\big] \!- (n-1)(\tilde{a} +\tilde{b})\E\big[\big(T^{(1)}\big)^{n-2}\big],\\
    & \phi_n^{(2)} \!=\! \E\big[X^{(2)}\big(T^{(2)}\big)^{n-1}\big] \!- (n-1)(\tilde{b} -\tilde{a})\E\big[\big(T^{(2)}\big)^{n-2}\big].\\
\end{align*}
In this step, we find the smallest $n \in \N$ such that at least one of the values $\{\phi_n^{(i)}\}_{i=1}^2$ is non-zero. 
We denote this value by $n^*$.
Note that such $n^*$ exists under the assumption imposed on $\eps_t$ in Theorem \ref{th:case 3} (see the proof for the details).

\textbf{Step 4.} 
Define $\psi_{ j}^{(1)}$ and $\psi_{ j}^{(2)}$ for every positive integer $j$ as follows. 
If $n^*$ is odd,
\begin{equation*}
    \psi_{ j}^{(i)} = \E\big[\big(T^{(i)}\big)^{n^*-j}\big(X^{(i)}\big)^j\big],
\end{equation*}
and if $n^*$ is even,
\begin{align*}
    & \psi_{ j}^{(1)} = \E\big[\big(X^{(1)}\big)^{n^*-j}\big(T^{(1)}\big)^{j}\big] \\&\hspace{5em}- (n^*-1)(\tilde{a} +\tilde{b})\E\big[\big(X^{(1)}\big)^{n^*-2}\big],\\
    & \psi_{j}^{(2)} = \E\big[\big(X^{(2)}\big)^{n^*-j}\big(T^{(2)}\big)^{j}\big] \\&\hspace{5em}- (n^*-1)(\tilde{b} -\tilde{a})\E\big[\big(X^{(2)}\big)^{n^*-2}\big].\\
\end{align*}

Finally, two possibilities may occur based on the values of $\phi_{n^*}^{(i)}$, which are dealt with in step $5$.

\textbf{Step 5.} \textit{Case 1}: $\phi_{n^*}^{(1)} - \phi_{n^*}^{(2)} \neq 0$. 
Choose $j=3, l=2$ if $n^*$ is odd, and $j=1, l=(n^*-1)$ otherwise. 
The absolute value of $a=\gamma^{(2)}-\gamma^{(1)}$ can be computed via
\begin{equation*}
    \vert a\vert = \left\vert\frac{\psi_{j}^{(1)} - \psi_{j}^{(2)}}{\phi_{n^*}^{(1)} - \phi_{n^*}^{(2)}}\right\vert^{1/l},
\end{equation*}
and as mentioned earlier, $a$ has the same sign as $\tilde{a}$.
As a result, $\beta + \gamma^{(1)}$ can be computed as $\beta + \gamma^{(1)} = \frac{1}{2}(2\beta + \gamma^{(1)} + \gamma^{(2)} - a)$. Finally, we recover $\beta$ via
\begin{equation*}
    \beta = \beta + \gamma^{(1)} - \textit{GetRatio}\left(\left(\beta+\gamma^{(1)}\right)T^{(1)} - Y^{(1)}, T^{(1)}\right).
\end{equation*}

\textbf{Step 5.} \textit{Case 2}: $\phi_{n^*}^{(1)} - \phi_{n^*}^{(1)} = 0$.
In this case, by definition of $n^*$, $\phi_{n^*}^{(1)} + \phi_{n^*}^{(1)} \neq 0$. 
We choose $j=2, l=1$ if $n^*$ is odd, and $j=1, l=(n^*-1)$ otherwise. 
In this case we can recover $b$ via
\begin{equation*}
    \vert b\vert = \left\vert\frac{\psi_{ j}^{(1)} + \psi_{j}^{(2)}}{\phi_{n^*}^{(1)} + \phi_{n^*}^{(2)}}\right\vert^{1/l},
\end{equation*}
and the fact the $b$ and $\tilde{b}$ have the same sign.
Finally, $\beta = \frac{1}{2}(2\beta + \gamma^{(1)} + \gamma^{(2)} - b)$.

\subsection{The case \texorpdfstring{$\F\big( \M^{(1)}, \M^{(2)}\big) = \{\alpha \}$}{F(M1, M2)=alpha}}
\label{sec:case alpha}
Here, we assume that only the causal effect of the latent variable on the treatment varies across environments. 
The SCM $\M^{(i)}$ corresponding to the environment $(i)$ is then
\begin{equation}
\begin{cases}
        & U^{(i)} := \epsilon_u, \\
        & T^{(i)} := \alpha^{(i)} U^{(i)} +\epsilon_{t}, \\
        & Y^{(i)} := \beta T^{(i)} + \gamma U^{(i)} + \epsilon_y.
\end{cases}  
\label{eq:case 4}
\end{equation}

\begin{restatable}{theorem}{thmcasefour}
    \label{th:case 4}
    Suppose $\M^{(1)}, \M^{(2)}$ are linear SCMs compatible with the DAG of Figure \ref{fig: main graph}, such that $\F\big( \M^{(1)}, \M^{(2)}\big)=\{\alpha\}$.
    Suppose $\exists n \in \N$ such that $\E\left[ \eps_u^n \right] \neq (n-1) \E\left[\eps_u^{n-2}\right] \E\left[ \eps_u^2\right]$.
    Then under Assumptions \ref{ass:linear SCM}-\ref{ass:distribution via moments} the treatment effect $\beta$ can be recovered uniquely almost surely\footnote{Here we consider the Lebesgue measure on the set of coefficients of linear SCMs $\M^{(1)}, \M^{(2)}$. Then the causal effect is not identifiable only for a set of coefficients with measure zero.}.
\end{restatable}
See Appendix \ref{apx:proofs} for a formal proof. 
The procedure for estimating the treatment effect $\beta$ is presented in Algorithm $\ref{alg:case 4}$.
This algorithm takes advantage of the fact that $Y^{(i)} - \hat{\beta} T^{(i)}$ does not depend on $\alpha^{(i)}$ if and only if $\hat{\beta}=\beta$. 
Therefore, solving the quadratic equation $h(\beta)$ (line 1 of Algorithm \ref{alg:case 4}) for $\beta$ (with coefficients that can be computed from observational data) provides us with two possible candidates.
We can then identify the true value of $\beta$ among these two candidates using the criterion in line 10 of Algorithm \ref{alg:case 4}, which is equal to zero only for the correct value of $\beta$.
See the proof of Theorem \ref{th:case 4} for further details.


\begin{algorithm}[t]
    \caption{$\F\big( \M^{(1)}, \M^{(2)}\big)=\{\alpha\}$}
    \label{alg:case 4}
    \textbf{Input:} $\{T^{(i)}, Y^{(i)}\}$ and $\F\big( \M^{(1)}, \M^{(2)}\big) = \{\eps_t\}$ 
    
    \begin{algorithmic}[1]
    \STATE $h(\beta) := \E\big[\big(Y^{(1)} - \beta T^{(1)}\big)^2\big] - \E\big[\big(Y^{(2)} - \beta T^{(2)}\big)^2\big]$
    \STATE $\beta_1, \beta_2 \gets roots\left(h(\cdot)\right)$
    \STATE $X_i^{(j)} := Y^{(j)} - \beta_i T^{(j)}$
    \STATE $\phi^{(j)}_{i, m} := \E\big[ \big(X^{(j)}_i\big)^{m} T^{(j)} \big]$
    \STATE $\psi^{(j)}_m :=  \E\big[ \big(X^{(j)}_i\big)^{m}\big]$
    \STATE $n_1 \gets 2$, $n_2 \gets 2$
    \WHILE{$\phi_{i, n_i-1}^{(1)} - (n_i-1)\phi_{i, 1}^{(1)}\psi^{(1)}_{n_i-2} = 0$}
        \STATE $n_{i} \gets n_{i}+1$
    \ENDWHILE
    \IF{$n_1=n_2$}
    \STATE $i \gets \underset{i}{\arg\,min} \left|\frac{\phi_{i, n-1}^{(1)}}{\phi_{i, n-1}^{(2)}} - \frac{\phi_{i, n-1}^{(1)} - (n-1)\phi_{i, 1}^{(1)}\psi^{(1)}_{n-2}}{\phi_{i, n-1}^{(2)} - (n-1)\phi_{i, 1}^{(2)}\psi^{(2)}_{n-2}}\right|$
    \ELSE
    \STATE $i \gets \underset{i}{\arg\,max} \left[ n_i \right]$
    \ENDIF
    \STATE \textbf{RETURN:} $\beta_i$
    \end{algorithmic}
\end{algorithm}

\subsection{Detecting the source of change}
\label{sec:source_of_change}
We begin by verifying whether the varying parameter is $\gamma$.
This can be done through comparing the distributions of $T$ and $Y$ between the environments:
if $\F\big( \M^{(1)}, \M^{(2)}\big) = \{\gamma\}$, then the distributions of $T^{(1)}$ and $T^{(2)}$ are identical, whereas the distributions of $Y^{(1)}$ and $Y^{(2)}$ are different.
If we conclude that $\F\big( \M^{(1)}, \M^{(2)}\big) \neq \{\gamma\}$, then we check whether $\F\big( \M^{(1)}, \M^{(2)}\big) = \{\alpha\}$. 
To do so we need to consider the following quantities.
\begin{align*}
    & \E\big[ T^{(1)} Y^{(1)} - T^{(2)} Y^{(2)}\big]\big/\E\big[ \big( T^{(1)} \big)^2 - \big( T^{(2)} \big)^2\big],\\
    & \E\big[ \big( Y^{(1)} \big)^2 - \big( Y^{(2)} \big)^2\big]\big/\E\big[ T^{(1)} Y^{(1)} - T^{(2)} Y^{(2)}\big].
\end{align*}
If these two quantities are equal or if both their denominators are equal to zero, we conclude that $\F\big( \M^{(1)}, \M^{(2)}\big) \neq \{\alpha\}$; otherwise $\F\big( \M^{(1)}, \M^{(2)}\big) = \{\alpha\}$.
Finally, if $\F\big( \M^{(1)}, \M^{(2)}\big) \not\subset \{\gamma, \alpha\}$, then we conclude that $\F\big( \M^{(1)}, \M^{(2)}\big) \subset \{\eps_t, \eps_u\}$.

\subsection{Non-identifiability}
\label{sec:non-ID-result}
We shall now show that $\beta$ is not identifiable if $\F\big( \M^{(1)}, \M^{(2)}\big) = \{\eps_u, \eps_t\}$. The SCM $\M^{(i)}$ corresponding to the environment $i$ takes the following form:
\begin{equation}
\M^{(i)} = 
\begin{cases}
        & U^{(i)} := \epsilon^{(i)}_u, \\
        & T^{(i)} := \alpha U^{(i)} +\epsilon^{(i)}_{t}, \\
        & Y^{(i)} := \beta T^{(i)} + \gamma U^{(i)} + \epsilon_y.
\end{cases}  
\label{eq:case non-identifiable}
\end{equation}
\begin{restatable}{theorem}{thmnonid}
    \label{th:non-identifiable beta}
     Suppose $\M^{(1)}, \M^{(2)}$ are linear SCMs compatible with the DAG of Figure \ref{fig: main graph}, such that $\F\big( \M^{(1)}, \M^{(2)}\big)=\{\eps_u, \eps_t\}$. The treatment effect $\beta$ is not identifiable from the observational data from both domains and $\F\big( \M^{(1)}, \M^{(2)}\big)$. 
\end{restatable}
\begin{proof}
    To prove that $\beta$ is not identifiable, we will construct two new SCMs $\tilde{\M}^{(1)}$ and $\tilde{\M}^{(2)}$,
    \begin{equation}
    \tilde{\M}^{(i)} = 
    \begin{cases}
            & \tilde{U}^{(i)} := \tilde{\epsilon}^{(i)}_u, \\
            & \tilde{T}^{(i)} := \tilde{\alpha} \tilde{U}^{(i)} +\tilde{\epsilon}^{(i)}_{t}, \\
            & \tilde{Y}^{(i)} := \tilde{\beta} \tilde{T}^{(i)} + \tilde{\gamma} \tilde{U}^{(i)} + \tilde{\epsilon}_y.
    \end{cases}  
    \label{eq:non-id:new SCMs}
    \end{equation}  
    such that $\F\big(\tilde{\M}^{(1)}, \tilde{\M}^{(2)}\big) = \{\tilde{\eps}_u, \tilde{\eps}_t\}$ and they induce the same observational distributions as $\M^{(1)}$ and $\M^{(2)}$, respectively, but the treatment effects differs for them from \eqref{eq:case non-identifiable}, i.e $\beta \neq \tilde{\beta}$. 
    To do so, we utilize the counter- example presented in \citep{salehkaleybar2020learning}. 
    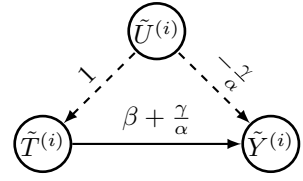
\begin{figure}[t]
      \begin{center}
        \begin{tikzpicture}[
                roundnode/.style={circle, draw=black!60,, fill=white, thick, inner sep=1pt},
                dashednode/.style = {circle, draw=black!60, dashed, fill=white, thick, inner sep=1pt},
                ]
                \node[node style]        (T)        at (-1.5, 0)                   {$\tilde{T}^{(i)}$};
                \node[node style]        (Y)        at (1.5, 0)                    {$\tilde{Y}^{(i)}$};
                \node[node style]       (U)        at (0, 1.5)                    {$\tilde{U}^{(i)}$};
                
                \draw[-latex, thick] (T.east) -- (Y.west) node[midway,sloped,above] {$\beta + \frac{\gamma}{\alpha}$} ;
                \draw[latex-, dashed, thick] (T) -- (U) node[midway,sloped,above] {$1$};
                \draw[latex-, dashed, thick] (Y) -- (U) node[midway,sloped,above] {$-\frac{\gamma}{\alpha}$};
            \end{tikzpicture}
      \end{center}
      \caption{The causal structure corresponding to SCM $\tilde{\M}^{(i)}$.}
      \label{fig: non-id structure}
    \end{figure}
    Specifically, the causal structure of the new models $\tilde{\mathcal{M}}^{(i)}$ corresponding to environments $i\in\{0,1\}$ can be seen in Figure \ref{fig: non-id structure} with the parameters defined as follows:
    \begin{align*}
        & \tilde{\eps}_u^{(i)} = \eps_t^{(i)},\; \tilde{\eps}_t^{(i)} = \alpha \eps_u^{(i)},\; \tilde{\eps}_y = \eps_y, \\
        & \tilde{\alpha} = 1,\; \tilde{\gamma} = -\frac{\gamma}{\alpha},\; \tilde{\beta} = \beta + \frac{\gamma}{\alpha}.
    \end{align*}
    Substituting these values into the set of equations $\ref{eq:non-id:new SCMs}$, we obtain
    \begin{equation*}
    \tilde{\M}^{(i)} = 
    \begin{cases}
            &\hspace{-.8em}\tilde{U}^{(i)} = \eps_t^{(i)}, \\
            &\hspace{-.8em}\tilde{T}^{(i)} = \eps_t^{(i)} + \alpha \eps_u^{(i)}, \\
            &\hspace{-.8em}\tilde{Y}^{(i)} = (\beta + \frac{\gamma}{\alpha}) (\eps_t^{(i)} + \alpha \eps_u^{(i)}) + -\frac{\gamma}{\alpha} \eps_t^{(i)} + \epsilon_y,
    \end{cases}  
    \end{equation*}
    and after regrouping and simplifications, it is easy to verify that
    \begin{align*}
        & \tilde{T}^{(i)} = \alpha \eps_u^{(i)} + \eps_t^{(i)} = T^{(i)}, \\
        & \tilde{Y}^{(i)} = (\alpha\beta + \gamma) \eps_u^{(i)} + \beta \eps_t^{(i)} + \epsilon_y = Y^{(i)},
    \end{align*}
    and that $\F\big(\tilde{\M}^{(1)}, \tilde{\M}^{(2)}\big) = \{\tilde{\eps}_u, \tilde{\eps}_t\}$, while $\beta\neq \tilde{\beta}$. This concludes the proof.
\end{proof}

\section{Related Work}
\label{sec:related_work}
In the multi-environment setting, there are two main lines of research in causality. The first  aims to learn an equivalent class of possible causal structures from samples collected across multiple environments. In this context, it is typically assumed that the distributional changes across environments arise due to interventions on exogenous noises or causal mechanisms, which are unknown to the observer. Various methods have been proposed to leverage these shifts for causal discovery, including constraint-based approaches \citep{ghassami2017learning,mooij2020joint,jaber2020causal,squires2020permutation,perry2022causal,zhou2022causal} or score-based methods \citep{brouillard2020differentiable,hagele2023bacadi,mameche2024learning}.

The second line of research focuses on identifying the direct causes of a target variable rather than inferring the entire causal structure. This approach is particularly relevant in settings where determining the parents of a specific variable is more critical than learning the whole causal graph. The primary objective here can be framed as a causal discovery task, which differs from causal effect identification, the main focus of our paper. However, this research direction is closely related to our work as it often assumes that causal mechanisms are linear (akin to us). Moreover, as part of the process to identify the parents of the target variable, the causal coefficients in the linear model are also often estimated. 

\begin{figure*}[t]
    \centering
    \begin{subfigure}[b]{0.485\textwidth}
        \centering
        \includegraphics{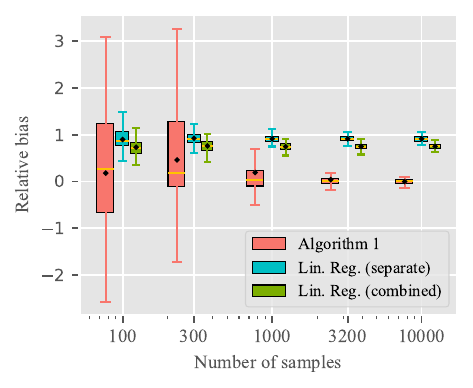}
        \caption{$\F\big( \M^{(1)}, \M^{(2)}\big) = \{\eps_t\}$}
        \label{fig:exp_eps_t}
    \end{subfigure}\hfill
    \begin{subfigure}[b]{0.485\textwidth}
        \centering
        \includegraphics{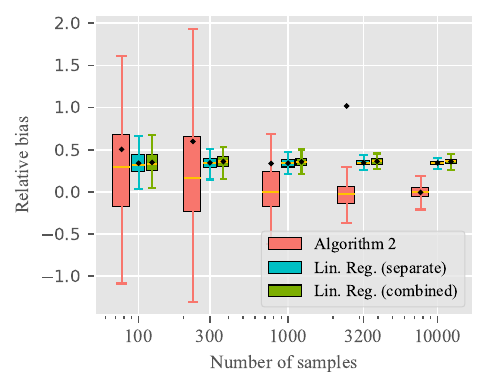}
        \caption{$\F\big( \M^{(1)}, \M^{(2)}\big) = \{\eps_u\}$}
        \label{fig:exp_eps_u}
    \end{subfigure}
    \begin{subfigure}[b]{0.485\textwidth}
        \centering
        \includegraphics{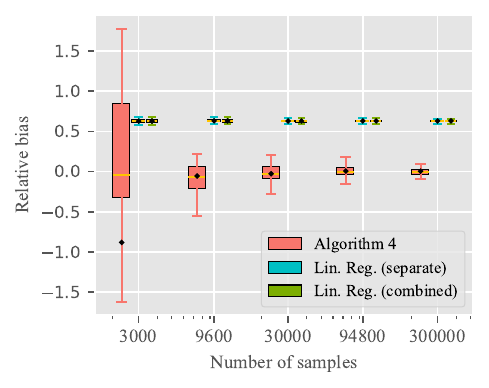}
        \caption{$\F\big( \M^{(1)}, \M^{(2)}\big) = \{\gamma\}$}
        \label{fig:exp_gamma}
    \end{subfigure}\hfill
    \begin{subfigure}[b]{0.485\textwidth}
        \centering
        \includegraphics{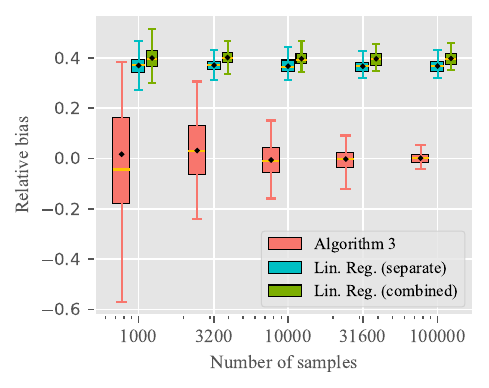}
        \caption{$\F\big( \M^{(1)}, \M^{(2)}\big) = \{\alpha\}$}
        \label{fig:exp_alpha}
    \end{subfigure}
    \caption{Relative estimation bias given data from two domains, when only $\epsilon_t$ (\ref{fig:exp_eps_t}), only $\epsilon_t$ (\ref{fig:exp_eps_u}), only $\gamma$ (\ref{fig:exp_gamma}), and only $\alpha$ (\ref{fig:exp_alpha}) varies across domains. }
    \label{fig:exponential}
\end{figure*}

As an example of the second line of research, \citet{peters2016causal} assumed that interventions could be applied to any variable except the target variable. They proposed the Invariant Causal Prediction (ICP) method, which leverages the invariance of the conditional distribution of the target variable across environments to identify a subset of its direct causes. Additionally, they assumed no latent confounding exists between the covariates and the target variable. Subsequent research has extended this idea to more general settings, including linear models with additive interventions, nonlinear models \citep{heinze2018invariant},  and sequential data \citep{pfister2019invariant}. For a comprehensive review of these developments, see \citep{buhlmann2020invariance}.

One of the drawbacks of the ICP and its extensions is  high computational cost, as these approaches must search over all possible subsets of covariates to verify the invariance of the conditional distribution of the target variable. To help address this issue, there has been growing interest in leveraging optimization techniques to recover direct causes in the multi-environment setting \citep{rothenhausler2019causal,gimenez2020identifying,yin2024optimization,wang2024causal}. 
For instance, for linear models, \citet{rothenhausler2019causal} proposed Causal Dantzig  exploiting ``invariant inner-products'' instead of the conditional invaraince in ICP. This method also allows latent confounding but not between covariates and the target variable. 
\citet{gimenez2020identifying} proposed KL regression to identify the direct causes of a target variable in the presence of latent confounding for linear models. In this approach, the model parameters are optimized by minimizing the Kullback-Leibler (KL) divergence between two multivariate Gaussian distributions, one corresponding to the observed covariance matrix and the other to the parameterized model—across different environments. The method assumes that all causal coefficients in the linear model remain unchanged, while only the covariance matrix of the covariates or the target variable may vary across environments. Therefore, the distribution of the latent confounder is required to remain invariant.  

Another related area of research is causal transportability \citep{bareinboim2014transportability, lee2020general}, which is to identify the distribution of a causal effect in a target domain using experimental data from a source domain and observational data from the target, typically under the assumption that certain mechanisms remain invariant across domains. In contrast, our work identifies the average causal effect of a treatment on an outcome using data collected from multiple environments, without access to any experimental interventions. Moreover, our results are derived under a specific causal graph structure, for which, to the best of our knowledge, no existing transportability method offers identifiability guarantees \citep{bareinboim2014transportability, lee2020general}.

In the context of robust prediction rather than for the problem of causal effect identification, \citet{arjovsky2019invariant} introduced Invariant Risk Minimization (IRM), which incorporates an additional penalty term into the empirical risk function to encourage invariance of the predictor across environments. Since its introduction, IRM has been extended to various domains, including meta-learning  \citep{bae2021meta}, reinforcement learning \citep{zhang2020invariant}, and causal inference \citep{shi2021invariant,lu2021nonlinear}. For linear models, \cite{rothenhausler2021anchor} proposed anchor regression as a robust predictor which is an interpolation between the solutions of ordinary least
squares and two-stage least square. It is noteworthy that the main goal in methods such as IRM or anchor regression is to have robust prediction against distribution shifts and they do not provide any guarantee for recovering the direct causes.

\section{Experimental Results}
We evaluated the performance of Algorithms \ref{alg:case 1} through \ref{alg:case 3}, in terms of \textit{relative estimation bias}, over a variety of settings. We define the relative bias as $(\frac{\hat{\beta}}{\beta}-1)$, where $\beta$ and $\hat{\beta}$ denote the true and estimated values of the parameter, respectively.
Our Python code is  \href{https://github.com/SinaAkbarii/IdentificationMultipleDomain}{accessible online}\footnote{https://github.com/SinaAkbarii/IdentificationMultipleDomain}.

Figure \ref{fig:exponential} illustrates the relative bias of our algorithms, and compares them to that of linear regression (ordinary least squares).
For linear regression, we included two versions: (i) \emph{separate}, which regresses the outcomes on the treatments separately in each domain, and takes the average of both estimates; and (ii) \emph{combined}, which concatenates the data from both domains and performs a single linear regression to estimate $\beta$.
Note that in our setup, most methods such as ICP reduce to linear regression.
We sampled the noise variables from an exponential distribution with parameter $\lambda$ chosen uniformly at random in the range $(0.9, 1.1)$.
In the cases where  $\eps_t$ or $\eps_u$ were changing between domains, we picked  $\lambda\in(0.45,0.55)$ as the alternate parameter.
Parameters $\alpha, \beta, \gamma$ were uniformly sampled from $(0.4, 0.6)$,
$(0.6, 0.7)$, and 
$(0.8, 0.9)$, respectively.
In the case  $\alpha$ or $\gamma$ were changing, the alternative values were sampled uniformly from $(0.8,0.9)$ and $(2,2.1)$, respectively.
Figures \ref{fig:exp_eps_t}, \ref{fig:exp_eps_u}, \ref{fig:exp_gamma} and \ref{fig:exp_alpha} represent the relative biases in estimation of $\beta$, when only one of $\eps_t$, $\eps_u$,  $\gamma$, or $\alpha$ varied across the domains.
The box plots show the median and $25\%$ quantiles of relative estimation bias.
As evident from all plots, our algorithms converge to the true parameter as the number of samples grows, whereas both linear regression baselines have a systematic bias regardless of the number of samples.
However, for smaller number of samples, our algorithms show a higher variability (in terms of sample variance) compared to linear regression. This is expected due to more complex estimation procedures in our algorithms.

In Appendix \ref{apx:experiment}, we present complementary simulation results for when noises are sampled according to various probability distributions.
Furthermore, Figures \ref{fig:hist_k_alg1} through \ref{fig:hist_n_alg4} in Appendix \ref{apx:experiment} depict some values of the parameters $k$ and $n$ which we observed running our algorithms. Interestingly, in Figure \ref{fig:hist_n_alg4}, for logistic distribution, the value of $n$ is not recovered correctly when the sample size is not large enough. Therefore, Algorithm \ref{alg:case 4} will use the incorrect estimation formula to compute $\beta$. This issue is reflected in the results of the Figure \ref{fig:gamma_logistic}. However, when the sample size increases, the correct value of $n$ is recovered, and in the next error boxes, it can be seen that our estimate is more accurate, unlike the one obtained by comparative methods. It is noteworthy that $n$ is often known in the literature of causal effect estimation via high-order moments (e.g. \citep{schkoda2024causal}), and to the best of our knowledge, our work is the first one that does not assume it in the experiments.
Additionally, the results show that for commonly encountered data distributions, the parameters  $k$ (or $n$) often are a small number.


\section{Conclusion}
We studied the problem of causal effect identification from observational data collected from two environments.
We showed that when there is a single unknown change across two domains, we can detect
whether  the causal effect of the latent confounder on the treatment or the outcome had changed or  the distribution of the exogenous noises of the treatment or latent confounder varied between the domains. We established that if the change occurs in the causal effect of latent confounder on other variables, the treatment effect is uniquely identifiable. Otherwise, it can be recovered up to two possible candidates. Additionally, we provided an estimation procedure tailored to each scenario and proved a non-identifiability result for identifying the treatment effect if the distribution of both exogenous noises corresponding to the treatment and latent varied across domains. 

For future work, we shall explore generalizing our findings to more than two environments. While using our current results for each pair of environments, we can establish identifiability in multiple environments in some cases, there may be situations where the treatment effect is not identifiable through pairwise comparisons but becomes identifiable  when more environments are considered.



\bibliography{uai2025}

\newpage

\onecolumn

\title{Causal Effect Identification in Heterogeneous Environments from Higher-Order Moments\\(Supplementary Material)}
\maketitle
\appendix
\section{Proofs}
\label{apx:proofs}
In this section we provide the proofs for the theorems introduced in main part.
\thmcaseone*
\begin{proof}
    Since $\alpha$ is not changing across the domains, the causal effect of the latent confounder $ U^{(i)}$ on $T^{(i)} $ can always be set to one by appropriately rescaling both the latent confounder and $ \gamma$. Hence, we can rewrite the structural equations in the two domains as follows:
    \begin{equation*}
    \M^{(i)} = 
    \begin{cases}
            & U^{(i)} := \epsilon^{(i)}_u, \\
            & T^{(i)} := U^{(i)} +\epsilon^{(i)}_{t}, \\
            & Y^{(i)} := \beta T^{(i)} + \gamma U^{(i)} + \epsilon^{(i)}_y.
    \end{cases}  
    \end{equation*}
    \textbf{Remark:} \textit{Here $\gamma$ and $\eps_u^{(1)}$ are already rescaled, and are actually different from the one in SCM \ref{eq:case 1}.} 
    
    Let $k$ be the smallest positive integer such that $\E\left[\left(\epsilon^{(1)}_{t}\right)^k\right]\neq \E\left[\left(\epsilon^{(2)}_{t}\right)^k\right]$. Note, that such $k$ will always exists, since otherwise $\eps_t^{(1)}$ and $\eps_t^{(2)}$ would be equal as distributions due to Assumption \ref{ass:distribution via moments}. Then,
    \begin{equation}
    \begin{split}
        \E\left[(T^{(1)})^k - (T^{(2)})^k\right]& = \E\left[\left(\epsilon_u^{(1)}+\epsilon_t^{(1)}\right)^k-\left(\epsilon_u^{(2)}+\epsilon_t^{(2)}\right)^k\right]\\
        &=\sum_{j=0}^k \binom{k}{j} \left(\E\left[\left(\epsilon_u^{(1)}\right)^j\right]\E\left[\left(\epsilon_t^{(1)}\right)^{k-j}\right]
        -  \E\left[\left(\epsilon_u^{(2)}\right)^j\right]\E\left[\left(\epsilon_t^{(2)}\right)^{k-j}\right]\right)\\
        &= \E\left[\left(\epsilon_t^{(1)}\right)^k\right] - \E\left[\left(\epsilon_t^{(2)}\right)^k\right],
    \end{split}
    \label{eq:Tk}
    \end{equation}
    where the third equality is due to the fact the difference term in the sum is equal to zero for $j\neq 0$ as the distribution of $\epsilon_u$ is not changing across the two domains and $\E\left[\left(\epsilon^{(1)}_{t}\right)^j\right]= \E\left[\left(\epsilon^{(2)}_{t}\right)^j\right]$ for $j<k$ due to the definition of $k$. Now, 
    
    \begin{equation}
        \begin{split}
            &\E\left[Y^{(1)}(T^{(1)})^{k-1}- Y^{(2)}(T^{(2)})^{k-1}\right]=\\
            &=\E\left[\left((\beta + \gamma)\epsilon^{(1)}_u + \beta \epsilon^{(1)}_{t} + \epsilon^{(1)}_y\right)\left(\epsilon_u^{(1)}+\epsilon_t^{(1)}\right)^{k-1}\right]-\E\left[\left((\beta + \gamma)\epsilon^{(2)}_u + \beta \epsilon^{(2)}_{t} + \epsilon^{(2)}_y\right)\left(\epsilon_u^{(2)}+\epsilon_t^{(2)}\right)^{k-1}\right]\\
            &\overset{(a)}{=}(\beta+\gamma)\sum_{j=0}^{k-1} \binom{k-1}{j} \left(\E\left[\left(\epsilon_u^{(1)}\right)^{j+1}\right]\E\left[\left(\epsilon_t^{(1)}\right)^{k-1-j}\right] - \E\left[\left(\epsilon_u^{(2)}\right)^{j+1}\right]\E\left[\left(\epsilon_t^{(2)}\right)^{k-1-j}\right]\right)\\
            &\hspace{1cm}+\beta\sum_{j=0}^{k-1} \binom{k-1}{j} \left(\E\left[\left(\epsilon_u^{(1)}\right)^{j}\right]\E\left[\left(\epsilon_t^{(1)}\right)^{k-j}\right] - \E\left[\left(\epsilon_u^{(2)}\right)^{j}\right]\E\left[\left(\epsilon_t^{(2)}\right)^{k-j}\right]\right)\\
            &\overset{(b)}{=}\beta \left(\E\left[\left(\epsilon_t^{(1)}\right)^k\right] - \E\left[\left(\epsilon_t^{(2)}\right)^k\right]\right),
        \end{split}
    \label{eq:ytk-1}
    \end{equation}
    where $(a)$ is based on the fact that $\E\left[\epsilon^{(i)}_y\left(\epsilon_u^{(i)}+\epsilon_t^{(i)}\right)^{k-1}\right]=\E\left[\epsilon^{(i)}_y\right]E\left[\left(\epsilon_u^{(i)}+\epsilon_t^{(i)}\right)^{k-1}\right]=0$ for $i\in \{1,2\}$ as the exogenous noises are independent and mean zero. Moreover, $(b)$ is due to the fact that all difference terms in the first sum are zero and in the second sum, only the difference for $j=0$ is nonzero because of the definition of $k$ and unchanging distribution of $\epsilon_u$ across the domains. Now, by dividing 
    \eqref{eq:ytk-1} by \eqref{eq:Tk},
    we have:
    \begin{equation}
        \beta= \frac{\E\left[Y^{(1)}(T^{(1)})^{k-1}- Y^{(2)}(T^{(2)})^{k-1}\right]}{\E\left[(T^{(1)})^k - (T^{(2)})^k\right]},
    \end{equation}
    which shows the causal effect of $T$ on $Y$ is identifiable from the two domains in case only the distribution of $\epsilon_t$ is changing across the domains.
\end{proof}

\thmcasetwo*
\begin{proof}
    Similar to the theorem \ref{th:case 1}, we write $\E\left[(T^{(1)})^k - (T^{(2)})^k\right]$ and $\E\left[Y^{(1)}(T^{(1)})^{k-1}- Y^{(2)}(T^{(2)})^{k-1}\right]$ based on the moments of exogenous noises. Herein, let $k$ be the smallest positive integer such that $\E\left[\left(\epsilon^{(1)}_{u}\right)^k\right]\neq \E\left[\left(\epsilon^{(2)}_{u}\right)^k\right]$. Then,
\begin{equation}
\begin{split}
    \E\left[(T^{(1)})^k - (T^{(2)})^k\right]& = \E\left[\left(\epsilon_u^{(1)}+\epsilon_t^{(1)}\right)^k-\left(\epsilon_u^{(2)}+\epsilon_t^{(2)}\right)^k\right]\\
    &=\sum_{j=0}^k \binom{k}{j} \left(\E\left[\left(\epsilon_u^{(1)}\right)^j\right]\E\left[\left(\epsilon_t^{(1)}\right)^{k-j}\right]
    -  \E\left[\left(\epsilon_u^{(2)}\right)^j\right]\E\left[\left(\epsilon_t^{(2)}\right)^{k-j}\right]\right)\\
    &= \E\left[\left(\epsilon_u^{(1)}\right)^k\right] - \E\left[\left(\epsilon_u^{(2)}\right)^k\right],
\end{split}
\label{eq:Uk}
\end{equation}
where the third equality is due to the fact the difference term in the sum is equal to zero for $j\neq k$ as the distribution of $\epsilon_u$ is not changing across the two domains and $\E\left[\left(\epsilon^{(1)}_{u}\right)^j\right]= \E\left[\left(\epsilon^{(2)}_{u}\right)^j\right]$ for $j<k$ due to the definition of $k$. Moreover,
\begin{equation}
    \begin{split}
        &\E\left[Y^{(1)}(T^{(1)})^{k-1}- Y^{(2)}(T^{(2)})^{k-1}\right]=\\
        &=\E\left[\left((\beta + \gamma)\epsilon^{(1)}_u + \beta \epsilon^{(1)}_{t} + \epsilon^{(1)}_y\right)\left(\epsilon_u^{(1)}+\epsilon_t^{(1)}\right)^{k-1}\right]-\E\left[\left((\beta + \gamma)\epsilon^{(2)}_u + \beta \epsilon^{(2)}_{t} + \epsilon^{(2)}_y\right)\left(\epsilon_u^{(2)}+\epsilon_t^{(2)}\right)^{k-1}\right]\\
        &\overset{(a)}{=}(\beta+\gamma)\sum_{j=0}^{k-1} \binom{k-1}{j} \left(\E\left[\left(\epsilon_u^{(1)}\right)^{j+1}\right]\E\left[\left(\epsilon_t^{(1)}\right)^{k-1-j}\right] - \E\left[\left(\epsilon_u^{(2)}\right)^{j+1}\right]\E\left[\left(\epsilon_t^{(2)}\right)^{k-1-j}\right]\right)\\
        &\hspace{1cm}+\beta\sum_{j=0}^{k-1} \binom{k-1}{j} \left(\E\left[\left(\epsilon_u^{(1)}\right)^{j}\right]\E\left[\left(\epsilon_t^{(1)}\right)^{k-j}\right] - \E\left[\left(\epsilon_u^{(2)}\right)^{j}\right]\E\left[\left(\epsilon_t^{(2)}\right)^{k-j}\right]\right)\\
        &\overset{(b)}{=}(\beta+\gamma) \left(\E\left[\left(\epsilon_u^{(1)}\right)^k\right] - \E\left[\left(\epsilon_u^{(2)}\right)^k\right]\right),
    \end{split}
\label{eq:ytk-1_u}
\end{equation}
where $(a)$ is based on the fact that $\E\left[\epsilon^{(i)}_y\left(\epsilon_u^{(i)}+\epsilon_t^{(i)}\right)^{k-1}\right]=\E\left[\epsilon^{(i)}_y\right]E\left[\left(\epsilon_u^{(i)}+\epsilon_t^{(i)}\right)^{k-1}\right]=0$ for $i\in \{1,2\}$ as the exogenous noises are independent and mean zero. Moreover, $(b)$ is due to the fact that all difference terms in the second sum are zero, and in the first sum, only the difference for $j=k$ is nonzero because of the definition of $k$ and unchanging distribution of $\epsilon_t$ across the domains. Therefore, based on \eqref{eq:Uk} and \eqref{eq:ytk-1_u}, we can obtain the value of $\beta+\gamma$ as follows:

\begin{equation}
    \beta+\gamma= \frac{\E\left[Y^{(1)}(D^{(1)})^{k-1}- Y^{(2)}(D^{(2)})^{k-1}\right]}{\E\left[(D^{(1)})^k - (D^{(2)})^k\right]}.
\end{equation}

Now, in any domain $i\in\{1,2\}$, consider the following two equations:
\begin{equation*}
    \begin{cases}
        & T^{(i)} := U^{(i)}  +\epsilon^{(i)}_{t} = \epsilon^{(i)}_u +\epsilon^{(i)}_{t}, \\
        & Y^{(i)}-(\beta+\gamma)T^{(i)}=-\gamma\epsilon_t^{(i)}+\epsilon_y^{(i)}.
    \end{cases}
    \end{equation*}
Utilizing the cross-moment approach \citep{kivva2024cross} and more specifically \cite{kivva2024cross}[Theorem 1], we can identify the value of $\gamma$ from $T^{(i)}$ and $Y^{(i)}-(\beta+\gamma)T^{(i)}$ in any domain $i\in \{1,2\}$ given the assumption on $\epsilon_t$ in theorem. Therefore, this finishes the proof that $\beta$ is uniquely identifiable.
\end{proof}

\thmcasethree*
\begin{proof}
    \begin{equation*}
    \M^{(i)}
    \begin{cases}
            & U^{(i)} := \epsilon^{(i)}_u, \\
            & T^{(i)} := U^{(i)} +\epsilon^{(i)}_{t}, \\
            & Y^{(i)} := \beta T^{(i)} + \gamma^{(i)} U^{(i)} + \epsilon^{(i)}_y,
    \end{cases}  
    \end{equation*}
    For ease of notation, we omit superscripts corresponding to the exogenous noises where this is not important, since their distributions remain the same across the domains.
    Then,
    \begin{align*}
        & \E\left[Y^{(2)}T^{(2)} - Y^{(1)}T^{(1)}\right] = \left(\gamma^{(2)} - \gamma^{(1)}\right)\E\left[\eps_u^2\right],\\
        & \E\left[\left(Y^{(2)}\right)^2 - \left(Y^{(1)}\right)^2\right] = \left(\left(\beta+\gamma^{(2)}\right)^2 - \left(\beta+\gamma^{(1)}\right)^2\right)\E\left[\eps_u^2\right] = \left(\gamma^{(2)}-\gamma^{(1)}\right)\left(2\beta + \gamma^{(1)} + \gamma^{(2)}\right)\E\left[\eps_u^2\right].
    \end{align*}
    Note that in the above equations all the terms corresponding to the exogenous noises of observed variables are canceled out. From the two equations, we can compute $2\beta + \gamma^{(1)} + \gamma^{(2)}$ from the observational distribtuion.
    
    Let us define the random variables $X^{(1)}$ and $X^{(2)}$ as follows:
    \begin{align*}
        & X^{(1)} = \left(2\beta + \gamma^{(1)} + \gamma^{(2)}\right)T^{(1)} - 2Y^{(1)} =  \left(2\beta + \gamma^{(1)} + \gamma^{(2)}\right)\left(\eps_u^{(1)} + \eps_t^{(1)}\right) - 2\left(\left(\beta + \gamma^{(1)}\right)\eps_u^{(1)} + \beta\eps_t^{(1)} + \eps_y^{(1)}\right),\\
        & X^{(2)} = \left(2\beta + \gamma^{(1)} + \gamma^{(2)}\right)T^{(2)} - 2Y^{(2)} =  \left(2\beta + \gamma^{(1)} + \gamma^{(2)}\right)\left(\eps_u^{(2)} + \eps_t^{(2)}\right) - 2\left(\left(\beta + \gamma^{(2)}\right)\eps_u^{(2)} + \beta\eps_t^{(2)} + \eps_y^{(2)}\right) .
    \end{align*}
    Let us define $a = \gamma^{(2)} - \gamma^{(1)}$, $b = \gamma^{(2)} + \gamma^{(1)}$, $\eps^{(i)} = -2\eps_y^{(i)}$. Then,
    \begin{align*}
        & X^{(1)}  = \left(\gamma^{(2)} - \gamma^{(1)}\right)\eps_u^{(1)} + \left(\gamma^{(1)}+\gamma^{(2)}\right)\eps_t^{(1)} - 2\eps_y^{(1)}
        = a\eps_u^{(1)} + b\eps_t^{(1)} + \eps^{(1)},\\
        & X^{(2)}  = \left(\gamma^{(1)} - \gamma^{(2)}\right)\eps_u^{(2)} + \left(\gamma^{(1)}+\gamma^{(2)}\right)\eps_t^{(2)} - 2\eps_y^{(2)} = -a\eps_u^{(2)} + b\eps_t^{(2)} + \eps^{(2)}.
    \end{align*}
    Let us consider the following expectations:
    \begin{align*}
        & \E\left[\left(T^{(1)}\right)^2\right] = \E\left[\left(T^{(2)}\right)^2\right] = \E\left[\eps_u^2\right] + \E\left[\eps_t^2\right], \\
        & \E\left[T^{(1)} X^{(1)}\right] = \left(\gamma^{(2)} - \gamma^{(1)}\right)\E\left[\eps_u^2\right] + \left(\gamma^{(1)} + \gamma^{(2)}\right) \E\left[\eps_t^2\right] = a\E\left[\eps_u^2\right] + b\E\left[\eps_t^2\right],\\
        & \E\left[T^{(2)} X^{(2)}\right] = \left(\gamma^{(1)} - \gamma^{(2)}\right)\E\left[\eps_u^2\right] + \left(\gamma^{(1)} + \gamma^{(2)}\right) \E\left[\eps_t^2\right] = -a\E\left[\eps_u^2\right] + b\E\left[\eps_t^2\right].
    \end{align*}
    The difference and sum of the expectations above give us:
    \begin{align}
        & \hat{a} := \frac{1}{2} \left(\E\left[T^{(1)}X^{(1)} - T^{(2)}X^{(2)}\right]\right) = \left(\gamma^{(2)} - \gamma^{(1)}\right) \E\left[\eps_u^2\right] = a \E\left[\eps_u^2\right], \label{eq: sign a}\\
        & \hat{b} := \frac{1}{2} \left(\E\left[T^{(1)}X^{(1)} - T^{(2)}X^{(2)}\right]\right) = \left(\gamma^{(1)}+\gamma^{(2)}\right) \E\left[\eps_t^2\right] = b \E\left[\eps_t^2\right]. \label{eq: sign b}
    \end{align}
    Note that from \eqref{eq: sign a}-\eqref{eq: sign b}, we can deduce the sign of $a$ and $b$.
    Now, we show how to obtain either $a$ or $b$. If we recover $b=\gamma^{(1)} + \gamma^{(2)}$, then from knowing $\beta + \gamma^{(1)} + \gamma^{(2)}$ we compute $\beta$. If we recover $a$, then from knowing $\beta + \gamma^{(1)} + \gamma^{(2)}$ we can compute $\beta + \gamma^{(i)}$. From this point, we can use the same method proposed in proof of theorem  \ref{th:case 2} to compute $\beta$.

    Consider,
    \begin{align}
        & \E\left[\left(X^{(1)}\right)^3\right] = a^3\E\left[\eps_u^3\right] + b^3\E\left[\eps_t^3\right] + \E\left[\eps^3\right], \\
        & \E\left[\left(X^{(2)}\right)^3\right] = -a^3\E\left[\eps_u^3\right] + b^3\E\left[\eps_t^3\right] + \E\left[\eps^3\right], \\
        & \implies\E\left[ \left(X^{(1)}\right)^3 - \left(X^{(2)}\right)^3\right] = a^3\E\left[\eps_u^3\right]. \label{eq: a^3 eps_u^3}
    \end{align}
    On the other hand,
    \begin{align}
        & \E\left[X^{(1)}\left(T^{(1)}\right)^2\right] = a\E\left[\eps_u^3\right] + b \E\left[\eps_t^3\right]\\
        & \E\left[X^{(2)} \left(T^{(2)}\right)^2\right] = -a\E\left[\eps_u^3\right] + b\E\left[\eps_t^3\right] \\
        & \E\left[ \left(X^{(1)}\right)\left(T^{(1)}\right)^2 - \left(X^{(2)}\right)\left(T^{(2)}\right)^2\right]= a\E\left[\eps_u^3\right]. \label{eq: a eps_u_3}
    \end{align}
    \textbf{Case 1:} $\E\left[\eps_u^3\right]\neq 0$. Then the ratio between \eqref{eq: a^3 eps_u^3} and \eqref{eq: a eps_u_3} gives $a^2$. We can identify the value of $a$ as we know the sign of $a$ from \eqref{eq: sign a}. 
    \\
    \textbf{Case 2:} $\E\left[\eps_u^3\right]=0$ and $\E\left[\eps_t^3\right]\neq0$. Then:
    \begin{align*}
        & \E\left[X^{(1)}\left(T^{(1)}\right)^2\right] = a\E\left[\eps_u^3\right] + b \E\left[\eps_t^3\right] = b\E\left[\eps_t^3\right], \\
        & \E\left[\left(X^{(1)}\right)^2T^{(1)}\right] = a^2\E\left[\eps_u^3\right] + b^2\E\left[\eps_t^3\right] = b^2\E\left[\eps_t^3\right].
    \end{align*}
    From the two above equations, we can compute the value of $b$.
    \\
    \textbf{Case 3:} $\E\left[\eps_u^3\right]=0$,  $\E\left[\epsilon_t^3\right]=0$ and $n\in \mathbb{N}$ - the smallest number such that one of the following equations hold:
    \begin{itemize}
        \item $\E\left[\eps_u^n\right]\neq(n-1)\E\left[\eps_u^{n-2}\right]\E\left[\eps_u^2\right]$.
        \item $\E\left[\eps_t^n\right]\neq(n-1)\E\left[\eps_t^{n-2}\right]\E\left[\eps_t^2\right]$.
    \end{itemize}
    Then
    \begin{align*}
        & \E\left[\left(T^{(1)}\right)^{n-1}X^{(1)}\right] = \E\left[\left(a\eps_u + b\eps_t + \eps\right)\left(\eps_u + \eps_t\right)^{n-1}\right] \\
        & \overset{*}{=} a\sum_{k=0}^{n-1}\binom{n-1}{k}\E\left[\eps_u^{k+1}\right]\E\left[\eps_t^{n-k-1}\right]  + b\sum_{k=0}^{n-1}\binom{n-1}{k}\E\left[\eps_{t}^{k+1}\right]\E\left[\eps_u^{n-k-1}\right]\\
        & \overset{**}{=} a\sum_{k=1}^{n-1}\binom{n-1}{k}\E\left[\eps_u^{k+1}\right]\E\left[\eps_t^{n-k-1}\right]  + b\sum_{k=1}^{n-1}\binom{n-1}{k}\E\left[\eps_{t}^{k+1}\right]\E\left[\eps_u^{n-k-1}\right],
    \end{align*}
    where $(*)$ and $(**)$ are based on the facts that exogenous noises are independent and have mean zero. 
    On the other hand,
    \begin{equation}
    \label{eq: (n-1)a eps_u^2 eps_u^(n-2)}
    \begin{split}
        & (n-1)a\E\left[\eps_u^2\right]\E\left[\left(T^{(1)}\right)^{n-2}\right] = (n-1)a\E\left[\eps_u^2\right]\E\left[\left(\eps_u + \eps_t\right)^{n-2}\right] \\
        & = (n-1)a\E\left[\eps_u^2\right]\sum_{k=0}^{n-2}\binom{n-2}{k}\E\left[\eps_u^k\right] \E\left[\eps_t^{n-k-2}\right] \\
        & = a\sum_{k=1}^{n-1} k \frac{(n-1)!}{k!(n-1-k)!}\E\left[\eps_u^2\right] \E\left[\eps_u^{k-1}\right]\E\left[\eps_t^{n-k-1}\right] \\
        & = a\sum_{k=1}^{n-1} k\binom{n-1}{k} \E\left[\eps_u^2\right] \E\left[\eps_u^{k-1}\right]\E\left[\eps_t^{n-k-1}\right].  
    \end{split}
    \end{equation}
    Note that $T^{(1)}$ is symmetric with respect to the exogenous noises $\eps_{u}^{(1)}$ and $\eps_{t}^{(1)}$. Therefore we can obtain similar equation to \eqref{eq: (n-1)a eps_u^2 eps_u^(n-2)}, where $\eps_{u}^{(1)}$ and $\eps_{t}^{(1)}$ are swapped. Hence, combining it with the knowledge that $\E\left[\eps_u^k\right]=(k-1)\E\left[\eps_u^{k-2}\right]\E\left[\eps_u^2\right]$ and $\E\left[\eps_t^k\right]=(k-1)\E\left[\eps_t^{k-2}\right]\E\left[\eps_t^2\right]$ for all $k<n$ we obtain
    \begin{equation}
    \label{eq: domain 1 main equation for a + b}
    \begin{split}
        & \E\left[\left(T^{(1)}\right)^{n-1}X^{(1)}\right] - (n-1)a\E\left[\eps_u^2\right]\E\left[\left(T^{(1)}\right)^{n-2}\right] - (n-1)b\E\left[\eps_t^2\right]\E\left[\left(T^{(1)}\right)^{n-2}\right]\\
        & = a\left(\E\left[\eps_u^{n}\right] - (n-1)\E\left[\eps_u^2\right] \E\left[\eps_u^{n-2}\right]\right) + b\left(\E\left[\eps_t^{n}\right] - (n-1)\E\left[\eps_t^2\right]\E\left[\eps_t^{n-2}\right]\right).
    \end{split}
    \end{equation}
    Similarly we obtain, 
    \begin{equation}
    \label{eq: domain 2 main equation for -a + b}
    \begin{split}
        & \E\left[X^{(2)}\left(T^{(2)}\right)^{n-1}\right] - (n-1)(-a)\E\left[\eps_u^2\right]\E\left[\left(T^{(2)}\right)^{n-2}\right] - (n-1)b\E\left[\eps_t^2\right]\E\left[\left(T^{(2)}\right)^{n-2}\right] \\
        & = -a\left(\E\left[\eps_u^{n}\right] - (n-1)\E\left[\eps_u^2\right]\E\left[\eps_u^{n-2}\right]\right) + b\left(\E\left[\eps_t^{n}\right] - (n-1)\E\left[\eps_t^2\right]\E\left[\eps_t^{n-2}\right]\right).
    \end{split}
    \end{equation}
    
    Now we will compute:
    \begin{align}
        \label{eq: first domain main case 3}
        & \E\left[\left(X^{(1)}\right)^{n-1}T^{(1)}\right] - (n-1)a\E\left[\eps_u^2\right]\E\left[\left(X^{(1)}\right)^{n-2}\right] - (n-1)b\E\left[\eps_t^2\right]\E\left[\left(X^{(1)}\right)^{n-2}\right] \\
        \label{eq: second domain main case 3}
        & \E\left[\left(X^{(2)}\right)^{n-1}T^{(2)}\right] - (n-1)(-a)\E\left[\eps_u^2\right]\E\left[\left(X^{(2)}\right)^{n-2}\right] - (n-1)b\E\left[\eps_t^2\right]\E\left[\left(X^{(2)}\right)^{n-2}\right]
    \end{align}
    
    For the $\E\left[\left(X^{(1)}\right)^{n-1}T^{(1)}\right]$, we have:
    \begin{align*}
        & \E\left[\left(X^{(1)}\right)^{n-1}T^{(1)}\right] = \E\left[\left(a\eps_u^{(1)} + b\eps_t^{(1)} + \eps^{(1)}\right)^{n-1}\left(\eps_u^{(1)}+\eps_t^{(1)}\right)\right] \\
        & \overset{*}{=} \E\left[\eps_u(a\eps_u + b\eps_t + \eps)^{n-1} + \eps_t(a\eps_u + b\eps_t + \eps)^{n-1}\right].
    \end{align*}
    Note that in $(*)$ we omit superscript $(1)$ since the moments of the exogenous noises are equal across domains. Then,
    \begin{align*}
        & \E\left[\eps_u(a\eps_u + b\eps_t + \eps)^{n-1}\right] = \sum_{m=0}^{n-1}\E\left[\eps^{m}\right] \binom{n-1}{m}\E\left[\eps_u(a\eps_u + b\eps_t)^{n-1-m}\right] \\
        & = \sum_{m=0}^{n-1} \E\left[\eps^{m}\right] \binom{n-1}{m} \sum_{k=1}^{n-m}a^{k-1}b^{n-m-k}\E\left[\eps_u^k\right]\E\left[\eps_t^{n-m-k}\right]\binom{n-1-m}{k-1}\\
        & \overset{*}{=} \sum_{m=0}^{n-2} \E\left[\eps^{m}\right] \binom{n-1}{m} \sum_{k=2}^{n-m}a^{k-1}b^{n-m-k}\E\left[\eps_u^k\right]\E\left[\eps_t^{n-m-k}\right]\binom{n-1-m}{k-1}.
    \end{align*}
    In the last equality $(*)$, the term in the second summation corresponding to $k=1$ is equal to zero due to the fact that exogenous noises have zero mean.
    
    For the $(n-1)a\E\left[\eps_u^2\right]\E\left[\right(X^{(1)}\left)^{n-2}\right]$ we have:
    \begin{align*}
        & (n-1)a\E\left[\eps_u^2\right]\E\left[\left(X^{(1)}\right)^{n-2}\right] = (n-1)a\E\left[\eps_u^2\right] \E\left[\left(a\eps_u + b\eps_t + \eps\right)^{n-2}\right]\\
        & = (n-1)\sum_{m=0}^{n-2}\E\left[\eps^{m}\right]\binom{n-2}{m}\sum_{k=0}^{n-m-2}a^{k+1}b^{n-m-k-2}\E\left[\eps_u^2\right]\E\left[\eps_u^{k}\right]\E\left[\eps_t^{n-m-k-2}\right]\binom{n-m-2}{k}\\
        & = (n-1)\sum_{m=0}^{n-2}\E\left[\eps^{m}\right]\binom{n-2}{m}\sum_{k=2}^{n-m}a^{k-1}b^{n-m-k}\E\left[\eps_u^2\right]\E\left[\eps_u^{k-2}\right]\E\left[\eps_t^{n-m-k}\right]\binom{n-m-2}{k-2}\\
        & = (n-1)\sum_{m=0}^{n-2}\E\left[\eps^{m}\right]\binom{n-1}{m}\frac{n-1-m}{n-1}\sum_{k=2}^{n-m}a^{k-1}b^{n-m-k}\E\left[\eps_u^2\right]\E\left[\eps_u^{k-2}\right]\E\left[\eps_t^{n-m-k}\right]\binom{n-m-1}{k-1}\frac{k-1}{n-1-m}\\
        & = \sum_{m=0}^{n-2}\E\left[\eps^{m}\right]\binom{n-1}{m}\sum_{k=2}^{n-m}a^{k-1}b^{n-m-k}(k-1)\E\left[\eps_u^2\right]\E\left[\eps_u^{k-2}\right]\E\left[\eps_t^{n-m-k}\right]\binom{n-m-1}{k-1}.
    \end{align*}
    Note that $\eps_t(a\eps_u+b\eps_t + \eps)^{n-1}$ can be obtained from $\eps_u(a\eps_u+b\eps_t + \eps)^{n-1}$ by substitutions $\eps_u \leftrightarrow \eps_t$ and $a \leftrightarrow b$. Hence we have:
    \begin{align*}
        & \E\left[\eps_t\left(a\eps_u + b\eps_t + \eps\right)^{n-1}\right]  = \sum_{m=0}^{n-2} \E\left[\eps^{m}\right] \binom{n-1}{m} \sum_{k=2}^{n-m}b^{k-1}a^{n-m-k}\E\left[\eps_t^k\right]\E\left[\eps_u^{n-m-k}\right]\binom{n-1-m}{k-1}.
    \end{align*}
    With similar logic,
    \begin{align*}
         (n-1)b\E\left[\eps_t^2\right]\E\Big[\big(X^{(1)}&\big)^{n-2}\Big] = \\&\sum_{m=0}^{n-2}\E\left[\eps^{m}\right]\binom{n-1}{m}\sum_{k=2}^{n-m}b^{k-1}a^{n-m-k}(k-1)\E\left[\eps_t^2\right]\E\left[\eps_t^{k-2}\right]\E\left[\eps_u^{n-m-k}\right]\binom{n-m-1}{k-1}.
    \end{align*}
    
    Note that $\E\left[\eps_{u}^{k} \right] = \left( k-1\right)\E\left[ \eps_u^{k-2}\right] \E \left[ \eps_u^2\right]$ and $\E\left[\eps_{t}^{k} \right] = \left( k-1\right)\E\left[ \eps_t^{k-2}\right] \E \left[ \eps_t^2\right]$ for all $k<n$. Therefore \eqref{eq: first domain main case 3} can be simplified as,
    \begin{equation}
    \label{eq: domain 1 main for a^(n-1) + b^(n-1)}
    \begin{split}
        & \E\left[\left(X^{(1)}\right)^{n-1}T^{(1)}\right] - (n-1)a\E\left[\eps_u^2\right]\E\left[\left(X^{(1)}\right)^{n-2}\right] - (n-1)b\E\left[\eps_t^2\right]\E\left[\left(X^{(1)}\right)^{n-2}\right] \\
        & = a^{n-1}\left(\E\left[\eps_u^{n}\right] - (n-1)\E\left[\eps_u^2\right]\E\left[\eps_u^{n-2}\right]\right) + b^{n-1}\left(\E\left[\eps_t^{n}\right] - (n-1)\E\left[\eps_t^2\right]\E\left[\eps_t^{n-2}\right]\right). 
    \end{split}
    \end{equation}
    Since the second domain can be obtained from the first by simple substitution of $a \leftrightarrow -a$, hence
    \begin{equation}
    \label{eq: domain 2 main for (-a)^(n-1) + b^(n-1)}
    \begin{split}
        & \E\left[\left(X^{(2)}\right)^{n-1}T^{(2)}\right] - (n-1)a\E\left[\eps_u^2\right]\E\left[\left(X^{(2)}\right)^{n-2}\right] - (n-1)b\E\left[\eps_t^2\right]\E\left[\left(X^{(2)}\right)^{n-2}\right] \\
        & = (-a)^{n-1}\left(\E\left[\eps_u^{n}\right] - (n-1)\E\left[\eps_u^2\right]\E\left[\eps_u^{n-2}\right]\right) + b^{n-1}\left(\E\left[\eps_t^{n}\right] - (n-1)\E\left[\eps_t^2\right]\E\left[\eps_t^{n-2}\right]\right). 
    \end{split}
    \end{equation}
    If $n$ is even then we can recover $a$ or $b$ from (\eqref{eq: domain 1 main for a^(n-1) + b^(n-1)} - \eqref{eq: domain 2 main for (-a)^(n-1) + b^(n-1)}) \textbackslash (\eqref{eq: domain 1 main equation for a + b} - \eqref{eq: domain 2 main equation for -a + b}) or (\eqref{eq: domain 1 main for a^(n-1) + b^(n-1)} + \eqref{eq: domain 2 main for (-a)^(n-1) + b^(n-1)}) \textbackslash (\eqref{eq: domain 1 main equation for a + b} + \eqref{eq: domain 2 main equation for -a + b}),respectively. 
    
    If $n$ is odd, then $\E[\eps_u^{k}] = 0$ and $\E[\eps_t^{k}] = 0$ for all $k$ natural odd numbers smaller than $n$. Hence,
    \begin{align*}
        & \E\left[\left(T^{(1)}\right)^{n-1}X^{(1)}\right] = a\E\left[\eps_u^n\right] + b\E\left[\eps_t^n\right],\\
        & \E\left[\left(T^{(2)}\right)^{n-1}X^{(2)}\right] = -a\E\left[\eps_u^n\right] + b\E\left[\eps_t^n\right],\\
        & \E\left[\left(T^{(1)}\right)^{n-3}\left(X^{(1)}\right)^3\right] - \E\left[\left(T^{(2)}\right)^{n-3}\left(X^{(2)}\right)^3\right] = 2a^3\E\left[\eps_u^n\right].\\
    \end{align*}
    The last equation is easy to verify, since all other terms of $\E\left[\left(T^{(1)}\right)^{n-3}\left(X^{(1)}\right)^3\right]$ except $a^3\E\left[\eps_u^n\right]$ are equal to zero or have identical one in $\E\left[\left(T^{(2)}\right)^{n-3}\left(X^{(2)}\right)^3\right]$. If $\E\left[\eps_u^n\right]\neq 0$ then we can compute $a$. In case when $\E\left[\eps_u^n\right]=0$ we additionally compute the following expressions,
    \begin{align*}
        & \E\left[\left(T^{(1)}\right)^{n-1}X^{(1)}\right] = a\E\left[\eps_u^n\right] + b\E\left[\eps_t^n\right],\\
        & \E\left[\left(T^{(1)}\right)^{n-2}\left(X^{(2)}\right)^2\right] = a^2\E\left[\eps_u^n\right] + b^2\E\left[\eps_t^n\right].
    \end{align*}
    Since $\E\left[\eps_u^n\right]=0$, then from the above equations we can recover $b$.
\end{proof}

\thmcasefour*
\begin{proof}
    \begin{equation*}
    \M^{(i)}
    \begin{cases}
            & U^{(i)} := \epsilon^{(i)}_u, \\
            & T^{(i)} := \alpha^{(i)} U^{(i)} +\epsilon^{(i)}_{t}, \\
            & Y^{(i)} := \beta T^{(i)} + \gamma U^{(i)} + \epsilon^{(i)}_y.
    \end{cases}  
    \end{equation*}

     Let us we consider the following quadratic equation with respect to parameter $\hat{\beta}$
    \begin{equation*}
        \E\left[ \left(Y^{(1)} - \hat{\beta} T^{(1)}\right)^2 \right] - \E\left[ \left(Y^{(2)} - \hat{\beta} T^{(2)}\right)^2 \right] = 0,
    \end{equation*}
    that simplifies as
    \begin{equation}
        \label{eq:case 4 quadratic eq}
        \E\left[\left(Y^{(1)}\right)^2 - \left(Y^{(2)}\right)^2 \right]\hat{\beta}^2 - 2\E\left[Y^{(1)}T^{(1)} - Y^{(2)}T^{(2)}\right]\hat{\beta} + \E\left[\left(T^{(1)}\right)^2 - \left(T^{(2)}\right)^2 \right]=0
    \end{equation}
    It easy to see that the following equations holds
    \begin{equation*}
        Y^{(1)} - \beta T^{(1)} = Y^{(2)} - \beta T^{(2)},
    \end{equation*}
    so $\beta$ will be one of the roots of the Eq. \eqref{eq:case 4 quadratic eq}. 
    
    Let us suppose $\beta^*$ is one of the roots of Eq. \eqref{eq:case 4 quadratic eq} and $X^{(1)}, X^{(2)}$ are defined as follows
    \begin{equation*}
        X^{(i)} = Y^{(i)} - \beta^* T^{(i)} = \left(\left(\beta - \beta^*\right)\alpha^{(i)} + \gamma\right)\eps_u^{(i)} + \left(\beta - \beta^*\right)\eps_t^{(i)} + \eps_y^{(i)}.
    \end{equation*}
    For simplicity of notation let us define $a^{(i)}:= \left(\beta - \beta^*\right)\alpha^{(i)} + \gamma$, $b:= \beta - \beta^*$, and so
    \begin{equation*}
        X^{(i)} = a^{(i)}\eps_u^{(i)} + b\eps_t^{(i)} + \eps_y^{(i)}.
    \end{equation*}
    Since $\E\left[\left(X^{(1)}\right)^2\right] = \E\left[\left(X^{(2)}\right)^2\right]$ it implies that $\left(a^{(1)}\right)^2 = \left(a^{(2)}\right)^2$. In case, when $\beta^* \neq \beta$ it only possible that \mbox{$a^{(1)}=-a^{(2)}$}, so 
    \begin{align}
        & \left(\beta - \beta^*\right)\alpha^{(1)} + \gamma = -\left( \left(\beta - \beta^*\right)\alpha^{(2)} + \gamma \right)\\
        & \Longrightarrow \alpha^{(1)} + \alpha^{(2)} = -2\frac{\gamma}{\beta - \beta^*} \neq 0
        \label{eq:case 4:a^(1) neq a^(2)}
    \end{align}

    Additionally, we have:
    \begin{equation}
    \label{eq:case4:X^{n-1}T}
    \begin{split}
        & \E\left[ \left(X^{(i)}\right)^{n-1} T^{(i)}\right] = \E \left[\left( \alpha^{(i)}\eps_u +  \eps_t\right)\left(a^{(i)}\eps_u + b\eps_t + \eps_y\right)^{n-1} \right]\\
        & =\alpha^{(i)}\E\left[\eps_u\left(a^{(i)}\eps_u + b\eps_t + \eps_y\right)^{n-1}\right] + \E\left[\eps_t\left(a^{(i)}\eps_u + b\eps_t + \eps_y\right)^{n-1}\right].
    \end{split}
    \end{equation}
    As it was done in the proof of theorem \ref{th:case 3} we can get 
    \begin{equation}
    \label{eq:case4:X^{n-1}T first}
    \begin{split}
        & \E\left[\eps_u(a^{(i)}\eps_u + b\eps_t + \eps_y)^{n-1}\right] = \sum_{m=0}^{n-1}\E\left[\eps_y^{m}\right] \binom{n-1}{m}\E\left[\eps_u(a^{(i)}\eps_u + b\eps_t)^{n-1-m}\right] \\
        & = \sum_{m=0}^{n-1} \E\left[\eps_y^{m}\right] \binom{n-1}{m} \sum_{k=1}^{n-m}(a^{(i)})^{k-1}b^{n-m-k}\E\left[\eps_u^k\right]\E\left[\eps_t^{n-m-k}\right]\binom{n-1-m}{k-1}\\
        & \overset{*}{=} \sum_{m=0}^{n-2} \E\left[\eps_y^{m}\right] \binom{n-1}{m} \sum_{k=2}^{n-m}(a^{(i)})^{k-1}b^{n-m-k}\E\left[\eps_u^k\right]\E\left[\eps_t^{n-m-k}\right]\binom{n-1-m}{k-1},
    \end{split}
    \end{equation}
    and 
    \begin{equation}
    \label{eq:case4:X^{n-2} first}
    \begin{split}
        & (n-1)a^{(i)}\E\left[\eps_u^2\right]\E\left[\left(X^{(i)}\right)^{n-2}\right] = (n-1)a^{(i)}\E\left[\eps_u^2\right] \E\left[\left(a^{(i)}\eps_u + b\eps_t + \eps_y\right)^{n-2}\right]\\
        & = (n-1)\sum_{m=0}^{n-2}\E\left[\eps_y^{m}\right]\binom{n-2}{m}\sum_{k=0}^{n-m-2}(a^{(i)})^{k+1}b^{n-m-k-2}\E\left[\eps_u^2\right]\E\left[\eps_u^{k}\right]\E\left[\eps_t^{n-m-k-2}\right]\binom{n-m-2}{k}\\
        & = (n-1)\sum_{m=0}^{n-2}\E\left[\eps_y^{m}\right]\binom{n-2}{m}\sum_{k=2}^{n-m}(a^{(i)})^{k-1}b^{n-m-k}\E\left[\eps_u^2\right]\E\left[\eps_u^{k-2}\right]\E\left[\eps_t^{n-m-k}\right]\binom{n-m-2}{k-2}\\
        & = (n-1)\sum_{m=0}^{n-2}\E\left[\eps_y^{m}\right]\binom{n-1}{m}\frac{n-1-m}{n-1}\sum_{k=2}^{n-m}(a^{(i)})^{k-1}b^{n-m-k}\E\left[\eps_u^2\right]\E\left[\eps_u^{k-2}\right]\E\left[\eps_t^{n-m-k}\right]\binom{n-m-1}{k-1}\frac{k-1}{n-1-m}\\
        & = \sum_{m=0}^{n-2}\E\left[\eps_y^{m}\right]\binom{n-1}{m}\sum_{k=2}^{n-m}(a^{(i)})^{k-1}b^{n-m-k}(k-1)\E\left[\eps_u^2\right]\E\left[\eps_u^{k-2}\right]\E\left[\eps_t^{n-m-k}\right]\binom{n-m-1}{k-1}.
    \end{split}
    \end{equation}
    Note that the similar equations can be obtained for $\E\left[\eps_t(a^{(i)}\eps_u + b\eps_t + \eps_y)^{n-1}\right]$ and $(n-1)b\E\left[\eps_t^2\right]\E\left[\left(X^{(i)}\right)^{n-2}\right]$ through the substitutions $\eps_u \leftrightarrow \eps_t$ and $a^{(i)} \leftrightarrow b$. Moreover,
    \begin{equation}
    \label{eq:case4:XT}
    \begin{split}
        \E\left[ X^{(i)} T\right] = \alpha^{(i)} a^{(i)} \E\left[\eps_u^2\right] + b\E\left[\eps_t^2\right]
    \end{split}
    \end{equation}
    
    Then combining the Equations \eqref{eq:case4:X^{n-1}T}-\eqref{eq:case4:XT} we obtain
    \begin{equation}
        \label{eq:case4:main property}
    \begin{split}
        & \Phi^{(i)}(\beta^{*}) := \E\left[\left(X^{(i)}\right)^{n-1} T^{(i)} \right] - (n-1)\E\left[ X^{(i)} T\right]\E\left[\left(X^{(i)}\right)^{n-2}\right] \\
        & = \alpha^{(i)}(a^{(i)})^{n-1}\left(\E\left[\eps_u^n\right] - (n-1)\E\left[\eps_u^{n-2}\right]\E\left[\eps_u^2\right] \right) + b^{n-1}\left(\E\left[\eps_t^n\right] - (n-1)\E\left[\eps_t^{n-2}\right]\E\left[\eps_t^2\right] \right)
    \end{split}
    \end{equation}
    Note that $b = 0$ for $\beta^* = \beta$. 
    
    Suppose that $n$ is the smallest natural number such that $\Phi^{(i)}(\beta^*)\neq 0$ for some $i$. This also implies that one of the following inequalities holds
    \begin{itemize}
        \item $\E\left[\eps_u^n\right]\neq(n-1)\E\left[\eps_u^{n-2}\right]\E\left[\eps_u^2\right]$,
        \item $\E\left[\eps_t^n\right]\neq(n-1)\E\left[\eps_t^{n-2}\right]\E\left[\eps_t^2\right]$.
    \end{itemize}
    Then there are possible the following cases.
    
    1. $\E\left[\eps_u^n\right] - (n-1)\E\left[\eps_u^{n-2}\right]\E\left[\eps_u^2\right] = 0$, then $\Phi^{(1)}(\beta^{*}) = \Phi^{(2)}(\beta^{*})\neq 0$. However the last equation for $\Phi^{(i)}$ can not happen if $\beta^* = \beta$. Indeed, if $\beta^* = \beta$ then
    \begin{equation*}
    \begin{split}
        & \Phi^{(i)}(\beta^{*}) = \alpha^{(i)}\gamma^{n-1}\left(\E\left[\eps_u^n\right] - (n-1)\E\left[\eps_u^{n-2}\right]\E\left[\eps_u^2\right] \right).
    \end{split}
    \end{equation*}
    Moreover,
    \begin{equation*}
        \Phi^{(1)}(\beta^{*}) -  \Phi^{(2)}(\beta^{*}) = 0 = \left(\alpha^{(1)} - \alpha^{2}\right)\gamma^{n-1}\left(\E\left[\eps_u^n\right] - (n-1)\E\left[\eps_u^{n-2}\right]\E\left[\eps_u^2\right] \right)
    \end{equation*}
    however, the right-hand side of the equation can not be zero. This follows from inequalities $\gamma \neq 0$, $\alpha^{(1)}\neq \alpha^{2}$ and $\Phi^{(2)}(\beta^{*})\neq 0$.
    Consequently, this means that if $\Phi^{(1)}(\beta^{*}) = \Phi^{(2)}(\beta^{*})\neq 0$ then we pick wrong $\beta*$ and we should pick another root of the quadratic equation as $\beta$.

    2. $\E\left[\eps_u^n\right] - (n-1)\E\left[\eps_u^{n-2}\right]\E\left[\eps_u^2\right] \neq 0$. Let us assume for a moment that $\beta^* = \beta$. Then, 
    \begin{align*}
        & X^{(i)} = \gamma \eps_u^{(i)} + \eps_y^{(i)} \Longrightarrow \E \left[ \left(X^{(i)}\right)^{n-1} T^{(i)} \right] \overset{(1)}{=} \E\left[ \alpha^{(i)}\eps_u\left(\gamma \eps_u + \eps_y \right)^{n-1}\right]\\
        & \Longrightarrow \frac{\E \left[ \left(X^{(1)}\right)^{n-1} T^{(1)} \right]}{\E \left[ \left(X^{(2)}\right)^{n-1} T^{(2)} \right]} = \frac{\alpha^{(1)}}{\alpha^{(2)}}
    \end{align*}
    and
    \begin{align*}
        & \Phi^{(i)}(\beta^{*}) = \alpha^{(i)}\gamma^{n-1}\left(\E\left[\eps_u^n\right] - (n-1)\E\left[\eps_u^{n-2}\right]\E\left[\eps_u^2\right] \right)\\
        & \Longrightarrow \frac{\Phi^{(1)}(\beta^{*})}{\Phi^{(2)}(\beta^{*})} = \frac{\alpha^{(1)}}{\alpha^{(2)}} = \frac{\E \left[ \left(X^{(1)}\right)^{n-1} T^{(1)} \right]}{\E \left[ \left(X^{(2)}\right)^{n-1} T^{(2)} \right]}.
    \end{align*}
    Moreover,
    \begin{align*}
        & \frac{\E \left[ X^{(1)} T^{(1)} \right]}{\E \left[ X^{(2)} T^{(2)} \right]} = \frac{\alpha^{(1)}}{\alpha^{(2)}}.
    \end{align*}

    Now let us consider the case when $\beta^* \neq \beta$. Then we have,
    \begin{gather*}
        \frac{\E \left[ X^{(1)} T^{(1)} \right]}{\E \left[ X^{(2)} T^{(2)} \right]} = \frac{a^{(1)}\alpha^{(1)}\eps_u^2 + b\eps_t^2}{a^{(2)}\alpha^{(2)}\eps_u^2 + b\eps_t^2}\\
        \frac{\Phi^{(1)}(\beta^{*})}{\Phi^{(2)}(\beta^{*})} = \frac{\alpha^{(1)}(a^{(1)})^{n-1}\left(\E\left[\eps_u^n\right] - (n-1)\E\left[\eps_u^{n-2}\right]\E\left[\eps_u^2\right] \right) + b^{n-1}\left(\E\left[\eps_t^n\right] - (n-1)\E\left[\eps_t^{n-2}\right]\E\left[\eps_t^2\right] \right)}{\alpha^{(2)}(a^{(2)})^{n-1}\left(\E\left[\eps_u^n\right] - (n-1)\E\left[\eps_u^{n-2}\right]\E\left[\eps_u^2\right] \right) + b^{n-1}\left(\E\left[\eps_t^n\right] - (n-1)\E\left[\eps_t^{n-2}\right]\E\left[\eps_t^2\right] \right)}.
    \end{gather*}
    Note that the equality
    \begin{equation}
        \label{eq:case 4: measure 0}
        \frac{\E \left[ X^{(1)} T^{(1)} \right]}{\E \left[ X^{(2)} T^{(2)} \right]} = \frac{\Phi^{(1)}(\beta^{*})}{\Phi^{(2)}(\beta^{*})}
    \end{equation}
    holds for the parameters $\alpha^{(1)}, \alpha^{(2)}, \gamma$ only for the set of Lebesgue measure zero. Indeed, an Eq \ref{eq:case 4: measure 0} is equivalent to
    \begin{equation*}
        \E \left[ X^{(1)} T^{(1)} \right]\Phi^{(2)}(\beta^{*}) - \E \left[ X^{(2)} T^{(2)} \right]\Phi^{(1)}(\beta^{*})=0
    \end{equation*}
    that can be considered as polynomial with respect to the parameter $a^{(1)}$. It is easy to see that the coefficient near the highest degree of $a^{(1)}$ is non-zero because $\E\left[\eps_u^n\right] - (n-1)\E\left[\eps_u^{n-2}\right]\E\left[\eps_u^2\right] \neq 0$, $a^{(1)}=-a^{(2)}$, $\alpha^{(1)}\neq \alpha^{(2)}$ and $\alpha^{(1)} + \alpha^{(2)} \neq 0$ (Eq. \eqref{eq:case 4:a^(1) neq a^(2)}).

    Consequently, by verifying whether the Eq. \eqref{eq:case 4: measure 0} holds we can conclude which one of the roots is the correct one.
\end{proof}

\thmunknownfactor*
\begin{proof}
    To prove this theorem we specify a step-by-step procedure that determines which of the parameters $\alpha, \gamma, \eps_u, \eps_t$ varies across domains $\M^{(1)}, \M^{(2)}$ under the assumption of infinite data.

    \textbf{Step 1.} First we show that, we can verify that $\F\big( \M^{(1)}, \M^{(2)}\big) = \{\gamma\}$. Indeed, in such a case, we can statistically test whether treatment $T$ and outcome $Y$ are different as distributions in these two distributions. Since we assume that $\eps_y^{(1)}, \eps_y^{(2)}$ are equal as distributions ($\eps_y^{(1)} \overset{d}{=} \eps_y^{(2)}$), then under Assumption \ref{ass:distribution via moments} the equalities $T^{(1)}\overset{d}{=} T^{(2)}$ and $Y^{(1)}\overset{d}{\not =} Y^{(2)}$ hold if and only if $\F\big( \M^{(1)}, \M^{(2)}\big) = \{\gamma\}$. In practice, to verify the equalities $T^{(1)}\overset{d}{=} T^{(2)}$ and $Y^{(1)}\overset{d}{\not =} Y^{(2)}$ we may use Kolmogorov-Smirnov test or any other statistical test appropriate for it.

    \textbf{Step 2.} Knowing that $\gamma^{(1)}=\gamma^{(2)}$ we introduce a test that determines whether $\F\big( \M^{(1)}, \M^{(2)}\big) = \{\alpha\}$. Let us consider the following quantities,
    \begin{align}
        & \E\left[ \left( T^{(i)} \right)^2 \right] = \E\left[\left(\alpha^{(i)}\right)^2\left(\eps^{(i)}_u\right)^2 + \left(\eps_t^{(i)}\right)^2 \right],\\
        & \E\left[ T^{(i)} Y^{(i)}\right] = \E\left[\left(\alpha^{(i)}\beta + \gamma \right)\left(\eps^{(i)}_u\right)^2 + \beta\left(\eps_t^{(i)}\right)^2\right],\\
        & \E\left[ \left( Y^{(i)} \right)^2 \right] = \E\left[\left(\alpha^{(i)}\beta + \gamma \right)^2\left(\eps^{(i)}_u\right)^2 + \beta^2\left(\eps_t^{(i)}\right)^2 + \eps_y^2\right].
    \end{align}
    If $\F\big( \M^{(1)}, \M^{(2)}\big) = \{\alpha\}$ then,
    \begin{align*}
        & \E\left[ \left( T^{(1)} \right)^2 - \left( T^{(2)} \right)^2\right] = \E \left[\left(\left(\alpha^{(1)}\right)^2 - \left(\alpha^{(1)}\right)^2\right)\eps_u^2\right],\\
        & \E\left[ T^{(1)} Y^{(1)} - T^{(2)} Y^{(2)}\right] = \E\left[\beta\left(\left(\alpha^{(1)}\right)^2 - \left(\alpha^{(2)}\right)^2 \right)\eps_u^2 + \gamma\left(\alpha^{(1)} - \alpha^{(2)} \right)\eps_u^2\right],\\
        & \E\left[ \left( Y^{(1)} \right)^2 - \left( Y^{(2)} \right)^2\right] = \E\left[\beta^2\left(\left(\alpha^{(1)}\right)^2 - \left(\alpha^{(2)}\right)^2 \right)\eps_u^2 + 2\beta\gamma\left(\alpha^{(1)} - \alpha^{(2)} \right)\eps_u^2\right].
    \end{align*}
    On the other hand, if $\F\big( \M^{(1)}, \M^{(2)}\big) = \{\eps_u\}$
    \begin{align*}
        & \E\left[ \left( T^{(1)} \right)^2 - \left( T^{(2)} \right)^2\right] = \E\left[\alpha\left( \left(\eps^{(1)}_u\right)^2 - \left(\eps^{(2)}_u\right)^2\right)\right],\\
        & \E\left[ T^{(1)} Y^{(1)} - T^{(2)} Y^{(2)}\right] = \E\left[\alpha\left(\alpha\beta + \gamma\right)\left( \left(\eps^{(1)}_u\right)^2 - \left(\eps^{(2)}_u\right)^2\right)\right],\\
        & \E\left[ \left( Y^{(1)} \right)^2 - \left( Y^{(2)} \right)^2\right] = \E\left[\left(\alpha\beta + \gamma\right)^2\left( \left(\eps^{(1)}_u\right)^2 - \left(\eps^{(2)}_u\right)^2\right)\right],
    \end{align*}
    and if $\F\big( \M^{(1)}, \M^{(2)}\big) = \{\eps_t\}$
    \begin{align*}
        & \E\left[ \left( T^{(1)} \right)^2 - \left( T^{(2)} \right)^2\right] = \E\left[ \left(\eps^{(1)}_t\right)^2 - \left(\eps^{(2)}_t\right)^2\right],\\
        & \E\left[ T^{(1)} Y^{(1)} - T^{(2)} Y^{(2)}\right] = \E\left[\beta\left( \left(\eps^{(1)}_t\right)^2 - \left(\eps^{(2)}_t\right)^2\right)\right],\\
        & \E\left[ \left( Y^{(1)} \right)^2 - \left( Y^{(2)} \right)^2\right] = \E\left[\beta^2\left( \left(\eps^{(1)}_t\right)^2 - \left(\eps^{(2)}_t\right)^2\right)\right].
    \end{align*}

    Note that for the $\F\big( \M^{(1)}, \M^{(2)}\big) = \{\alpha\}$ it is easy to see that at least one of the quantities $\E\left[ \left( T^{(1)} \right)^2 - \left( T^{(2)} \right)^2\right]$ or $\E\left[ T^{(1)} Y^{(1)} - T^{(2)} Y^{(2)}\right]$ is not equal to zero. Moreover,
    \begin{equation*}
        \frac{\E\left[ T^{(1)} Y^{(1)} - T^{(2)} Y^{(2)}\right]}{\E\left[ \left( T^{(1)} \right)^2 - \left( T^{(2)} \right)^2\right]} \neq \frac{\E\left[ \left( Y^{(1)} \right)^2 - \left( Y^{(2)} \right)^2\right]}{\E\left[ T^{(1)} Y^{(1)} - T^{(2)} Y^{(2)}\right]}.
    \end{equation*}
    However for the case $\F\big( \M^{(1)}, \M^{(2)}\big) = \{\eps_t\}$ or $\F\big( \M^{(1)}, \M^{(2)}\big) = \{\eps_u\}$ either the following equation holds
    \begin{equation*}
        \E\left[ \left( T^{(1)} \right)^2 - \left( T^{(2)} \right)^2\right] = \E\left[ T^{(1)} Y^{(1)} - T^{(2)} Y^{(2)}\right] = 0,
    \end{equation*}
    or 
    \begin{equation*}
        \frac{\E\left[ T^{(1)} Y^{(1)} - T^{(2)} Y^{(2)}\right]}{\E\left[ \left( T^{(1)} \right)^2 - \left( T^{(2)} \right)^2\right]} = \frac{\E\left[ \left( Y^{(1)} \right)^2 - \left( Y^{(2)} \right)^2\right]}{\E\left[ T^{(1)} Y^{(1)} - T^{(2)} Y^{(2)}\right]}.
    \end{equation*}
    Since both of these equations are impossible for the case $\F\big( \M^{(1)}, \M^{(2)}\big) = \{\alpha\}$, then we can use them for the verification procedure.

    Now, knowing that $\F\big( \M^{(1)}, \M^{(2)}\big) \in \{\eps_u, \eps_t\}$ we will show that it is impossible to identify $\beta$ uniquely. To prove it, it is enough to consider the similar construction of models $\M^{(i)}$ and $\hat{\M}^{(i)}$ presented in the proof of Theorem \ref{th:non-identifiable beta}. Indeed, since we do not know which parameter of the parameters $\eps_u$ or $\eps_t$ may vary across the environments, then both of the models are possible SCMs that concludes the proof.
\end{proof}

\begin{proposition}
    \label{prop:non-id eps_y}
    Suppose $\M^{(1)}, \M^{(2)}$ are linear SCMs compatible with the DAG of Figure \ref{fig: main graph}, such that $\F\big( \M^{(1)}, \M^{(2)}\big)=\{\eps_y\}$. Then treatment causal effect $\beta$ is not identifiable.
\end{proposition}
\begin{proof}
    To prove that $\beta$ is not identifiable, we will construct two new SCMs $\tilde{\M}^{(1)}$ and $\tilde{\M}^{(2)}$,
    \begin{equation}
    \tilde{\M}^{(i)} = 
    \begin{cases}
            & \tilde{U}^{(i)} := \tilde{\epsilon}_u, \\
            & \tilde{T}^{(i)} := \tilde{\alpha} \tilde{U}^{(i)} +\tilde{\epsilon}_{t}, \\
            & \tilde{Y}^{(i)} := \tilde{\beta} \tilde{T}^{(i)} + \tilde{\gamma} \tilde{U}^{(i)} + \tilde{\epsilon}^{(i)}_y.
    \end{cases}  
    \label{eq:non-id:new SCMs app}
    \end{equation}  
    such that $\F\big(\tilde{\M}^{(1)}, \tilde{\M}^{(2)}\big) = \{\tilde{\eps}_y\}$ and they induce the same observational distributions as $\M^{(1)}$ and $\M^{(2)}$, respectively, but the treatment effects are different, i.e $\beta \neq \tilde{\beta}$. 
    To do so, we again utilize the counter- example presented in \citep{salehkaleybar2020learning}. 

    \begin{align*}
        & \tilde{\eps}_u = \eps_t,\; \tilde{\eps}_t = \alpha \eps_u,\; \tilde{\eps}_y^{(i)} = \eps_y^{(i)}, \\
        & \tilde{\alpha} = 1,\; \tilde{\gamma} = -\frac{\gamma}{\alpha},\; \tilde{\beta} = \beta + \frac{\gamma}{\alpha}.
    \end{align*}
    Substituting these values into the set of equations $\ref{eq:non-id:new SCMs app}$, we obtain
    \begin{equation*}
    \tilde{\M}^{(i)} = 
    \begin{cases}
            &\hspace{-.8em}\tilde{U}^{(i)} = \eps_t, \\
            &\hspace{-.8em}\tilde{T}^{(i)} = \eps_t + \alpha \eps_u, \\
            &\hspace{-.8em}\tilde{Y}^{(i)} = (\beta + \frac{\gamma}{\alpha}) (\eps_t + \alpha \eps_u) + -\frac{\gamma}{\alpha} \eps_t + \epsilon^{(i)}_y,
    \end{cases}  
    \end{equation*}
    and after regrouping and simplifications, it is easy to verify that
    \begin{align*}
        & \tilde{T}^{(i)} = \alpha \eps_u + \eps_t = T^{(i)}, \\
        & \tilde{Y}^{(i)} = (\alpha\beta + \gamma) \eps_u + \beta \eps_t + \epsilon^{(i)}_y = Y^{(i)},
    \end{align*}
    and $\F\big(\tilde{\M}^{(1)}, \tilde{\M}^{(2)}\big) = \{\tilde{\eps}_y\}$. This concludes the proof. 
\end{proof}

\newpage
\section{Omitted Pseudo-code}
We present the pseudo-code pertaining to the estimation procedure of $\beta$ when $\gamma$ changes across domains, which was omitted from the main text due to space limitations.
\begin{algorithm}[ht]
    \caption{$\F\big( \M^{(1)}, \M^{(2)}\big)=\{\gamma\}$}
    \label{alg:case 3}\textbf{Input:} $\{T^{(i)}, Y^{(i)}\}$ and $\F\big( \M^{(1)}, \M^{(2)}\big) = \{\gamma\}$
    
    \begin{algorithmic}[1]
        \STATE $r\gets \dfrac{\E\left[\left(Y^{(2)}\right)^2 - \left(Y^{(1)}\right)^2\right]}{\E\left[Y^{(2)}T^{(2)} - Y^{(1)}T^{(1)}\right]}\quad\quad$ \COMMENT{$r=2\beta + \gamma^{(1)}+\gamma^{(2)}$}
        \STATE $X^{(i)}\gets rT^{(i)} - 2Y^{(i)}$
        \STATE $\tilde{a}\gets \frac{1}{2}\big(\E\big[T^{(1)}X^{(1)}\big]-\E\big[T^{(2)}X^{(2)}\big]\big)$, $\tilde{b}\gets \frac{1}{2}\big(\E\big[T^{(1)}X^{(1)}\big]+\E\big[T^{(2)}X^{(2)}\big]\big)$
        \STATE $n^*\gets 2$
        \WHILE{$\phi_n^{(1)}=0$ and $\phi_n^{(2)}=0$}
            \STATE $n^*\gets n^*+1$
        \ENDWHILE
        \IF{$\phi_n^{(1)}-\phi_n^{(2)}\neq0$}
            \IF{$n^*$ is odd}
                \STATE $j\gets 3$, $l\gets 2$
            \ELSE
                \STATE $j\gets 1$, $l\gets (n^*-1)$
            \ENDIF
            \STATE $a \gets sign(\tilde{b}) \left\vert\frac{\psi_{ j}^{(1)} - \psi_{j}^{(2)}}{\phi_{n^*}^{(1)} - \phi_{n^*}^{(2)}}\right\vert^{1/l}\quad\quad$ \COMMENT{$a=\gamma^{(2)}-\gamma^{(1)}$}
            \STATE $\tilde{r}\gets \frac{1}{2}(r-a)\quad\quad$ \COMMENT{$\tilde{r}=\beta+\gamma^{(1)}$} 
            \STATE $\beta \gets \tilde{r} - \textit{GetRatio}\big(\tilde{r}T^{(1)}-Y^{(1)}, T^{(1)}\big)$
        \ELSE
            \IF{$n^*$ is odd}
                \STATE $j\gets 2$, $l\gets 1$
            \ELSE
                \STATE $j\gets 1$, $l\gets (n^*-1)$
            \ENDIF
            \STATE $b \gets sign(\tilde{b}) \left\vert\frac{\psi_{ j}^{(1)} + \psi_{j}^{(2)}}{\phi_{n^*}^{(1)} + \phi_{n^*}^{(2)}}\right\vert^{1/l}\quad\quad$ \COMMENT{$b=\gamma^{(1)}+\gamma^{(2)}$}
            \STATE $\beta\gets \frac{1}{2}(r-b)$
        \ENDIF
        \RETURN{$\beta$}
    \end{algorithmic}
\end{algorithm}
\newpage
    

\section{Complementary Experimental Results}\label{apx:experiment}
\begin{figure*}[h]
    \centering
    \begin{subfigure}[b]{0.485\textwidth}
        \centering
        \includegraphics{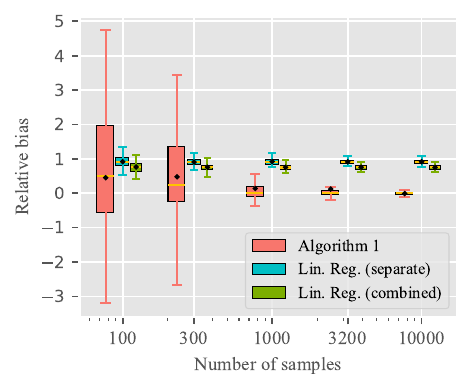}
        \caption{$\F\big( \M^{(1)}, \M^{(2)}\big) = \{\eps_t\}$}
        \label{fig:gamma_eps_t}
    \end{subfigure}\hfill
    \begin{subfigure}[b]{0.485\textwidth}
        \centering
        \includegraphics{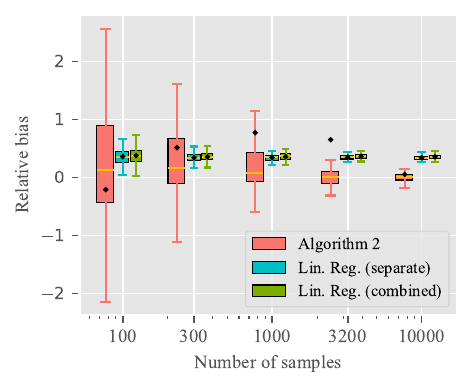}
        \caption{$\F\big( \M^{(1)}, \M^{(2)}\big) = \{\eps_u\}$}
        \label{fig:gamma_eps_u}
    \end{subfigure}
    \begin{subfigure}[b]{0.485\textwidth}
        \centering
        \includegraphics{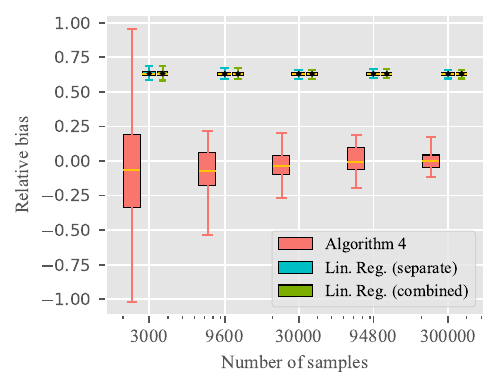}
        \caption{$\F\big( \M^{(1)}, \M^{(2)}\big) = \{\gamma\}$}
        \label{fig:gamma_gamma}
    \end{subfigure}\hfill
    \begin{subfigure}[b]{0.485\textwidth}
        \centering
        \includegraphics{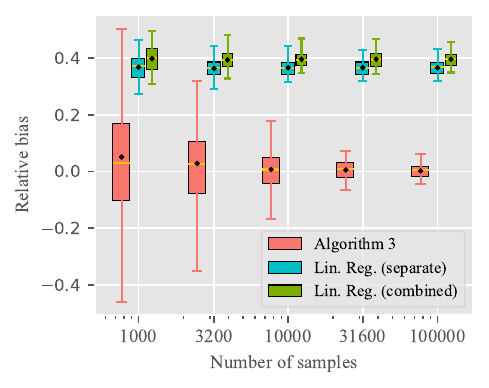}
        \caption{$\F\big( \M^{(1)}, \M^{(2)}\big) = \{\alpha\}$}
        \label{fig:gamma_alpha}
    \end{subfigure}
    \caption{Relative estimation bias given data from two domains, when only $\epsilon_t$ (\ref{fig:gamma_eps_t}), only $\epsilon_t$ (\ref{fig:gamma_eps_u}), only $\gamma$ (\ref{fig:gamma_gamma}), and only $\alpha$ (\ref{fig:gamma_alpha}) varies across domains.
    Noise variables are sampled from a Gamma distribution.}
\end{figure*}

\begin{figure*}[t]
    \centering
    \begin{subfigure}[b]{0.485\textwidth}
        \centering
        \includegraphics{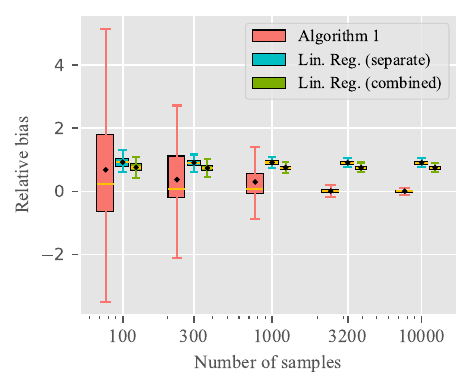}
        \caption{$\F\big( \M^{(1)}, \M^{(2)}\big) = \{\eps_t\}$}
        \label{fig:eps_t_gumbel}
    \end{subfigure}\hfill
    \begin{subfigure}[b]{0.485\textwidth}
        \centering
        \includegraphics{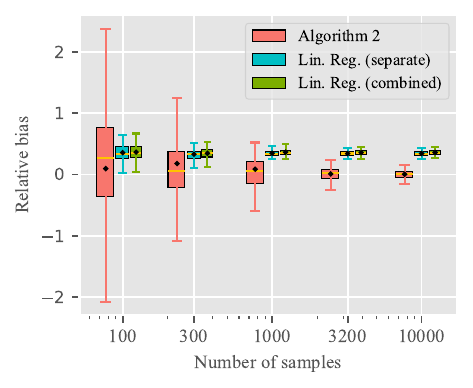}
        \caption{$\F\big( \M^{(1)}, \M^{(2)}\big) = \{\eps_u\}$}
        \label{fig:eps_u_gumbel}
    \end{subfigure}
    \begin{subfigure}[b]{0.485\textwidth}
        \centering
        \includegraphics{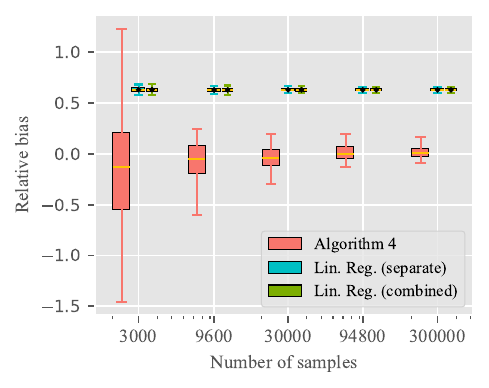}
        \caption{$\F\big( \M^{(1)}, \M^{(2)}\big) = \{\gamma\}$}
        \label{fig:gamma_gumbel}
    \end{subfigure}\hfill
    \begin{subfigure}[b]{0.485\textwidth}
        \centering
        \includegraphics{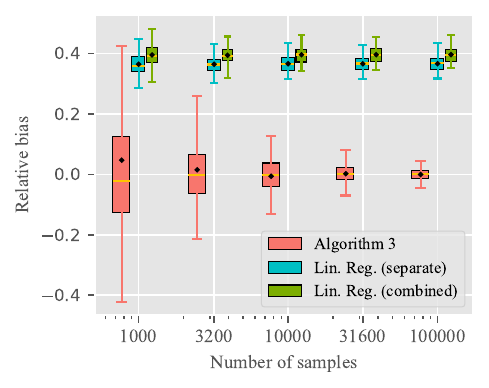}
        \caption{$\F\big( \M^{(1)}, \M^{(2)}\big) = \{\alpha\}$}
        \label{fig:alpha_gumbel}
    \end{subfigure}
    \caption{Relative estimation bias given data from two domains, when only $\epsilon_t$ (\ref{fig:eps_t_gumbel}), only $\epsilon_t$ (\ref{fig:eps_u_gumbel}), only $\gamma$ (\ref{fig:gamma_gumbel}), and only $\alpha$ (\ref{fig:alpha_gumbel}) varies across domains.
    Noise variables are sampled from a Gumbel distribution.}
\end{figure*}

\begin{figure*}[t]
    \centering
    \begin{subfigure}[b]{0.485\textwidth}
        \centering
        \includegraphics{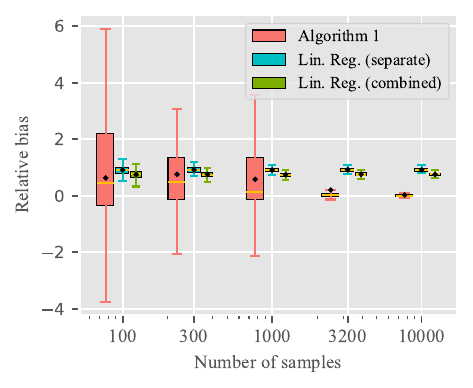}
        \caption{$\F\big( \M^{(1)}, \M^{(2)}\big) = \{\eps_t\}$}
        \label{fig:eps_t_logistic}
    \end{subfigure}\hfill
    \begin{subfigure}[b]{0.485\textwidth}
        \centering
        \includegraphics{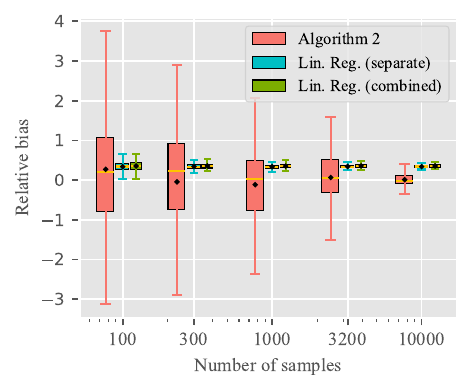}
        \caption{$\F\big( \M^{(1)}, \M^{(2)}\big) = \{\eps_u\}$}
        \label{fig:eps_u_logistic}
    \end{subfigure}
    \begin{subfigure}[b]{0.485\textwidth}
        \centering
        \includegraphics{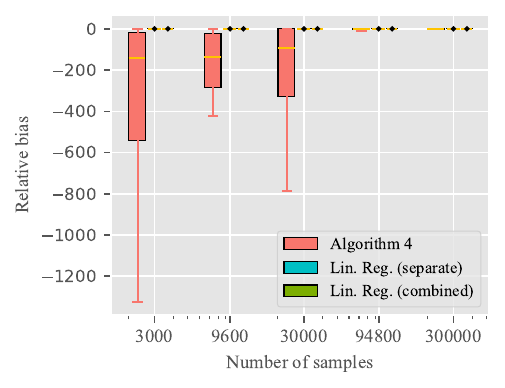}
        \caption{$\F\big( \M^{(1)}, \M^{(2)}\big) = \{\gamma\}$}
        \label{fig:gamma_logistic}
    \end{subfigure}\hfill
    \begin{subfigure}[b]{0.485\textwidth}
        \centering
        \includegraphics{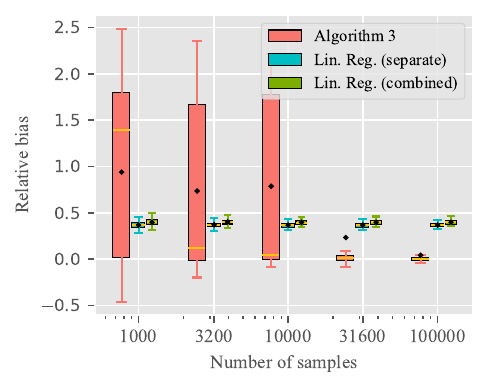}
        \caption{$\F\big( \M^{(1)}, \M^{(2)}\big) = \{\alpha\}$}
        \label{fig:alpha_logistic}
    \end{subfigure}
    \caption{Relative estimation bias given data from two domains, when only $\epsilon_t$ (\ref{fig:eps_t_logistic}), only $\epsilon_t$ (\ref{fig:eps_u_logistic}), only $\gamma$ (\ref{fig:gamma_logistic}), and only $\alpha$ (\ref{fig:alpha_logistic}) varies across domains.
    Noise variables are sampled from a Logistic distribution.}
\end{figure*}

\begin{figure*}[t]
    \centering
    \begin{subfigure}[b]{0.485\textwidth}
        \centering
        \includegraphics{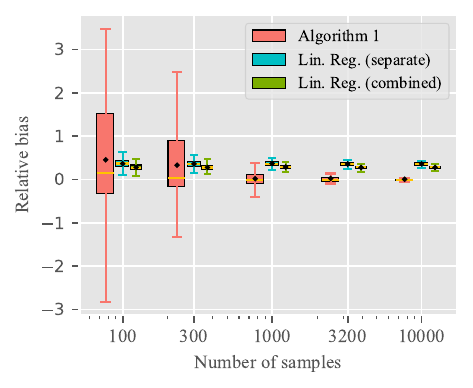}
        \caption{$\F\big( \M^{(1)}, \M^{(2)}\big) = \{\eps_t\}$}
        \label{fig:eps_t_exponential_logistic}
    \end{subfigure}\hfill
    \begin{subfigure}[b]{0.485\textwidth}
        \centering
        \includegraphics{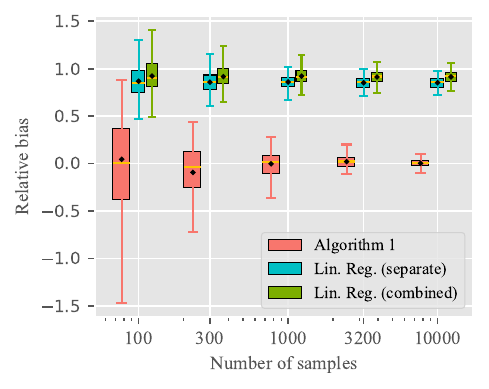}
        \caption{$\F\big( \M^{(1)}, \M^{(2)}\big) = \{\eps_u\}$}
        \label{fig:eps_u_exponential_logistic}
    \end{subfigure}
    \caption{Relative estimation bias given data from two domains, when only $\epsilon_t$ (\ref{fig:eps_t_exponential_logistic}), and only $\epsilon_t$ (\ref{fig:eps_u_exponential_logistic}) varies across domains.
    All noise variables are sampled from an exponential distribution, except the alternating noise variable in the second domain which is sampled from a logistic distribution.}
\end{figure*}

\begin{figure*}[t]
    \centering
    \begin{subfigure}[b]{0.485\textwidth}
        \centering
        \includegraphics{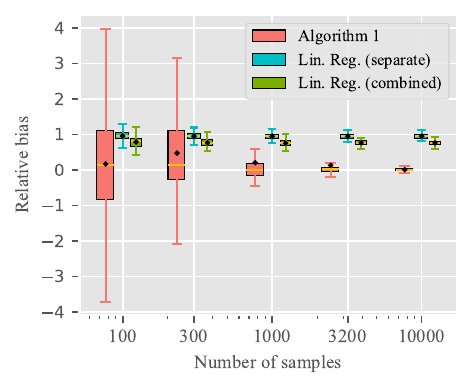}
        \caption{$\F\big( \M^{(1)}, \M^{(2)}\big) = \{\eps_t\}$}
        \label{fig:eps_t_gamma_uniform}
    \end{subfigure}\hfill
    \begin{subfigure}[b]{0.485\textwidth}
        \centering
        \includegraphics{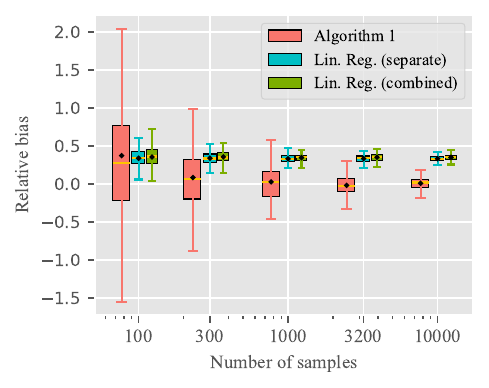}
        \caption{$\F\big( \M^{(1)}, \M^{(2)}\big) = \{\eps_u\}$}
        \label{fig:eps_u_gamma_uniform}
    \end{subfigure}
    \caption{Relative estimation bias given data from two domains, when only $\epsilon_t$ (\ref{fig:eps_t_gamma_uniform}), and only $\epsilon_t$ (\ref{fig:eps_u_gamma_uniform}) varies across domains.
    All noise variables are sampled from a Gamma distribution, except the alternating noise variable in the second domain which is sampled from a uniform distribution.}
\end{figure*}

\begin{figure*}[t]
    \centering
    \begin{subfigure}[b]{0.485\textwidth}
        \centering
        \includegraphics{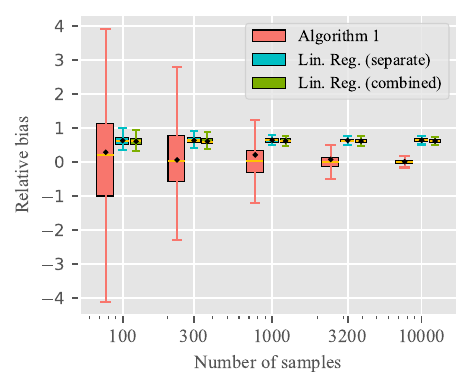}
        \caption{$\F\big( \M^{(1)}, \M^{(2)}\big) = \{\eps_t\}$}
        \label{fig:eps_t_gumbel_exponential}
    \end{subfigure}\hfill
    \begin{subfigure}[b]{0.485\textwidth}
        \centering
        \includegraphics{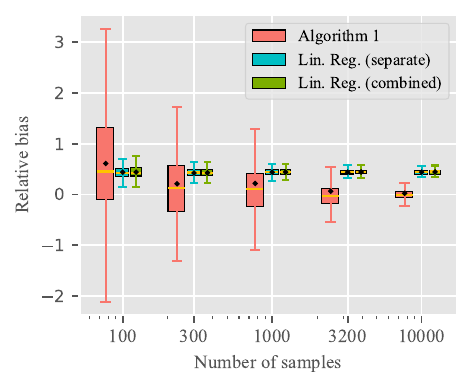}
        \caption{$\F\big( \M^{(1)}, \M^{(2)}\big) = \{\eps_u\}$}
        \label{fig:eps_u_gumbel_exponential}
    \end{subfigure}
    \caption{Relative estimation bias given data from two domains, when only $\epsilon_t$ (\ref{fig:eps_t_gumbel_exponential}), and only $\epsilon_t$ (\ref{fig:eps_u_gumbel_exponential}) varies across domains.
    All noise variables are sampled from a Gumbel distribution, except the alternating noise variable in the second domain which is sampled from an exponential distribution.}
\end{figure*}

\begin{figure*}[t]
    \centering
    \begin{subfigure}[b]{0.485\textwidth}
        \centering
        \includegraphics{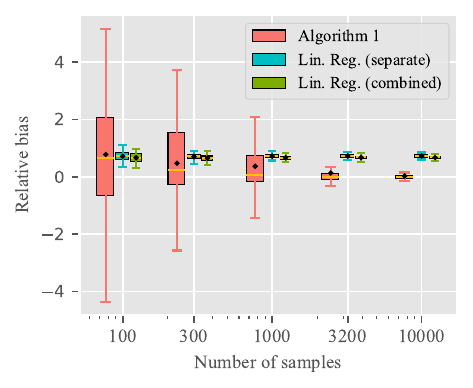}
        \caption{$\F\big( \M^{(1)}, \M^{(2)}\big) = \{\eps_t\}$}
        \label{fig:eps_t_logistic_gamma}
    \end{subfigure}\hfill
    \begin{subfigure}[b]{0.485\textwidth}
        \centering
        \includegraphics{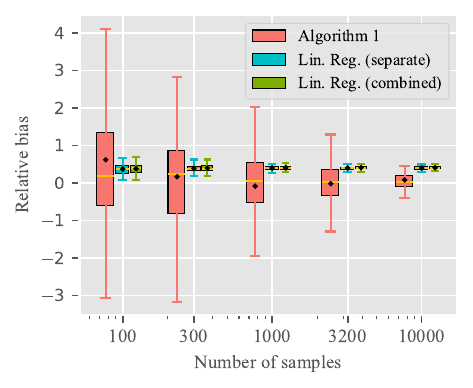}
        \caption{$\F\big( \M^{(1)}, \M^{(2)}\big) = \{\eps_u\}$}
        \label{fig:eps_u_logistic_gamma}
    \end{subfigure}
    \caption{Relative estimation bias given data from two domains, when only $\epsilon_t$ (\ref{fig:eps_t_logistic_gamma}), and only $\epsilon_t$ (\ref{fig:eps_u_logistic_gamma}) varies across domains.
    All noise variables are sampled from a logistic distribution, except the alternating noise variable in the second domain which is sampled from a Gamma distribution.}
\end{figure*}

\begin{figure*}
    \centering
    \begin{subfigure}[b]{0.485\textwidth}
        \centering
        \includegraphics{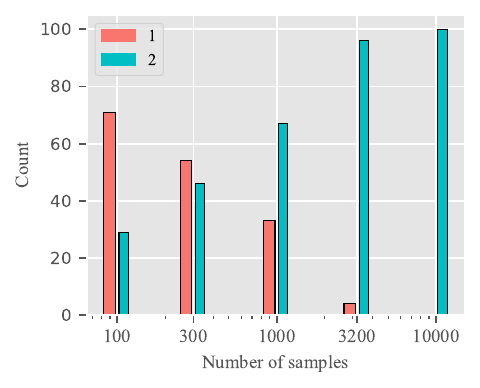}
        \caption{Exponential distribution}
    \end{subfigure}\hfill
    \begin{subfigure}[b]{0.485\textwidth}
        \centering
        \includegraphics{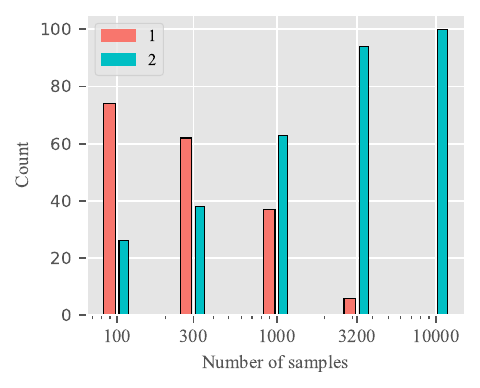}
        \caption{Gamma distribution}
    \end{subfigure}
    \begin{subfigure}[b]{0.485\textwidth}
        \centering
        \includegraphics{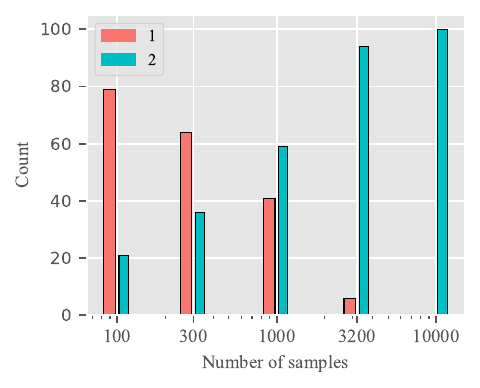}
        \caption{Gumbel distribution}
    \end{subfigure}\hfill
    \begin{subfigure}[b]{0.485\textwidth}
        \centering
        \includegraphics{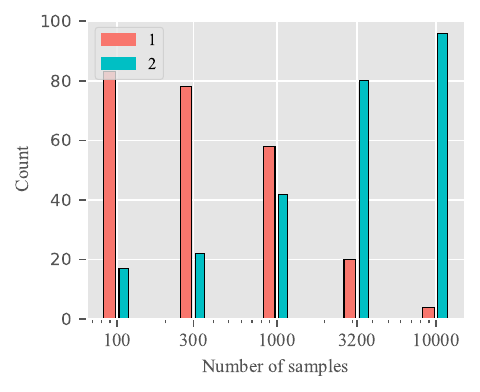}
        \caption{Logistic distribution}
    \end{subfigure}
    \caption{Histogram of $k$ in Algorithm 1.}
    \label{fig:hist_k_alg1}
\end{figure*}

\begin{figure*}
    \centering
    \begin{subfigure}[b]{0.485\textwidth}
        \centering
        \includegraphics{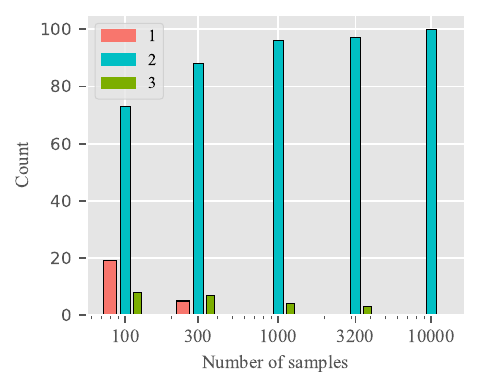}
        \caption{Exponential distribution}
    \end{subfigure}\hfill
    \begin{subfigure}[b]{0.485\textwidth}
        \centering
        \includegraphics{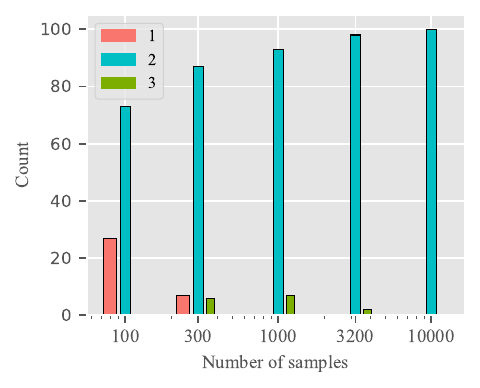}
        \caption{Gamma distribution}
    \end{subfigure}
    \begin{subfigure}[b]{0.485\textwidth}
        \centering
        \includegraphics{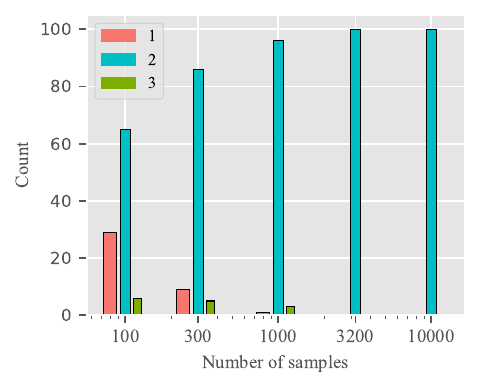}
        \caption{Gumbel distribution}
    \end{subfigure}\hfill
    \begin{subfigure}[b]{0.485\textwidth}
        \centering
        \includegraphics{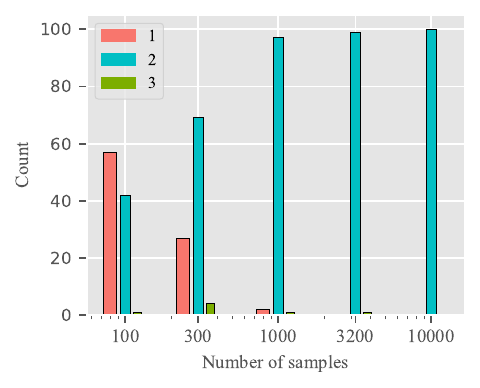}
        \caption{Logistic distribution}
    \end{subfigure}
    \caption{Histogram of $k$ in Algorithm 2.}
\end{figure*}

\begin{figure*}
    \centering
    \begin{subfigure}[b]{0.485\textwidth}
        \centering
        \includegraphics{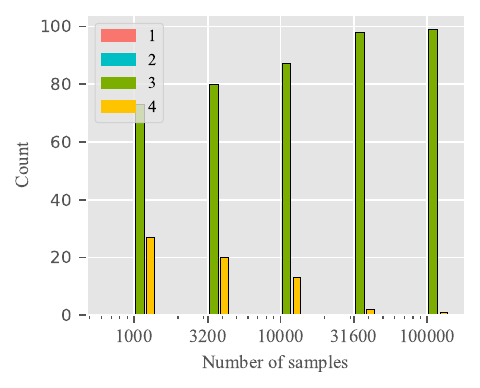}
        \caption{Exponential distribution}
    \end{subfigure}\hfill
    \begin{subfigure}[b]{0.485\textwidth}
        \centering
        \includegraphics{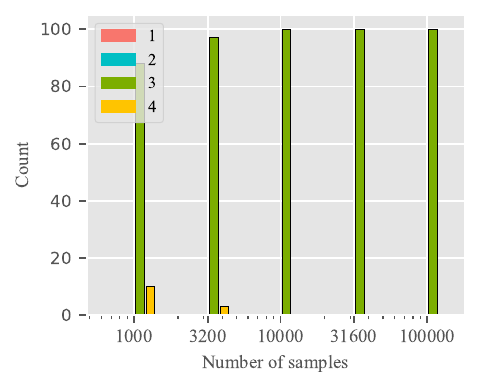}
        \caption{Gamma distribution}
    \end{subfigure}
    \begin{subfigure}[b]{0.485\textwidth}
        \centering
        \includegraphics{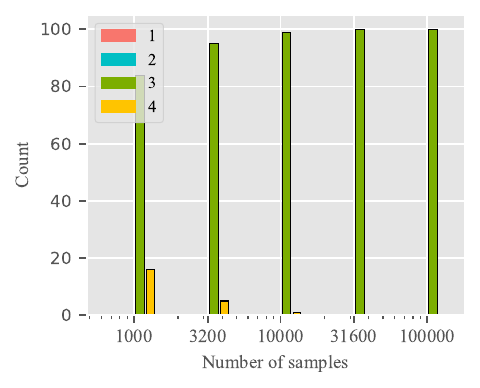}
        \caption{Gumbel distribution}
    \end{subfigure}\hfill
    \begin{subfigure}[b]{0.485\textwidth}
        \centering
        \includegraphics{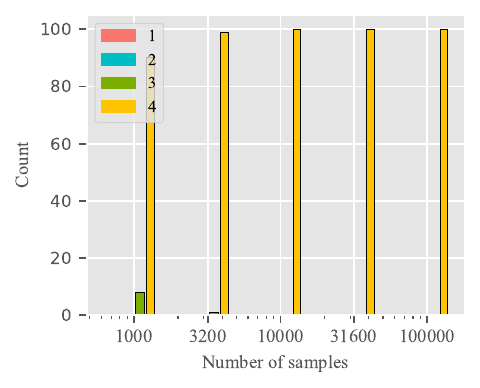}
        \caption{Logistic distribution}
    \end{subfigure}
    \caption{Histogram of $n_1$ in Algorithm 3.}
\end{figure*}

\begin{figure*}
    \centering
    \begin{subfigure}[b]{0.485\textwidth}
        \centering
        \includegraphics{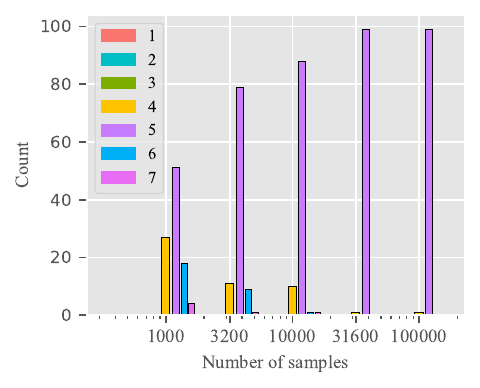}
        \caption{Exponential distribution}
    \end{subfigure}\hfill
    \begin{subfigure}[b]{0.485\textwidth}
        \centering
        \includegraphics{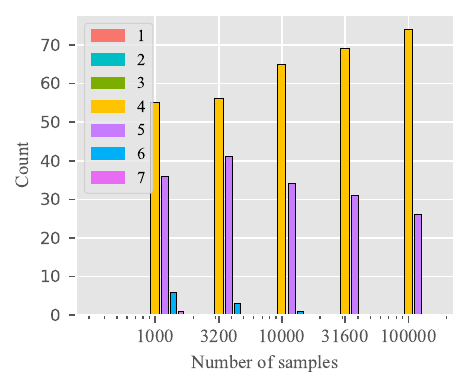}
        \caption{Gamma distribution}
    \end{subfigure}
    \begin{subfigure}[b]{0.485\textwidth}
        \centering
        \includegraphics{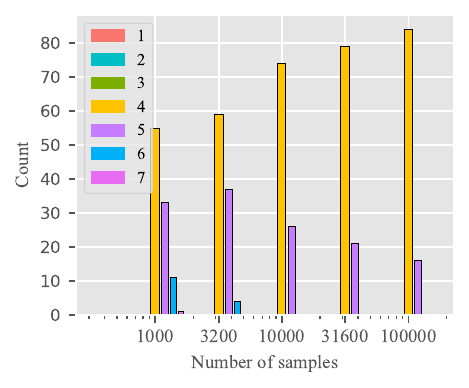}
        \caption{Gumbel distribution}
    \end{subfigure}\hfill
    \begin{subfigure}[b]{0.485\textwidth}
        \centering
        \includegraphics{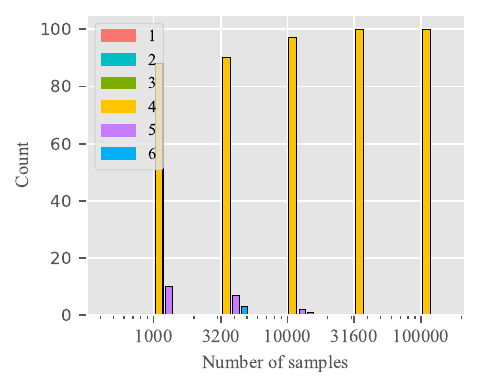}
        \caption{Logistic distribution}
    \end{subfigure}
    \caption{Histogram of $n_2$ in Algorithm 3.}
\end{figure*}

\begin{figure*}
    \centering
    \begin{subfigure}[b]{0.485\textwidth}
        \centering
        \includegraphics{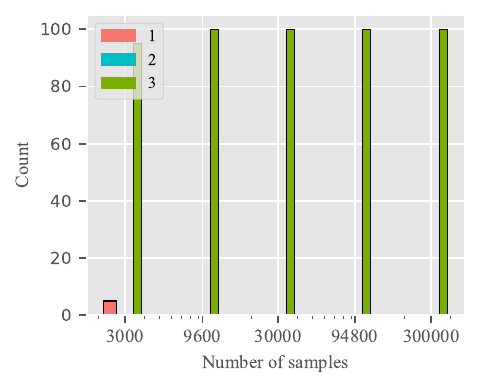}
        \caption{Exponential distribution}
    \end{subfigure}\hfill
    \begin{subfigure}[b]{0.485\textwidth}
        \centering
        \includegraphics{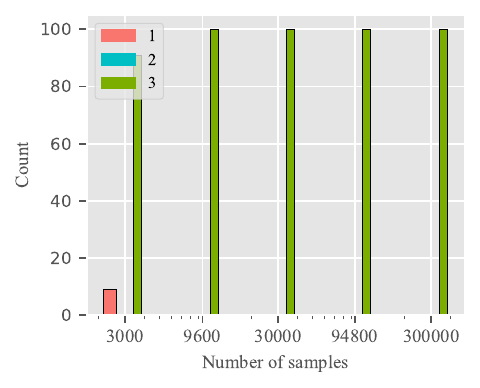}
        \caption{Gamma distribution}
    \end{subfigure}
    \begin{subfigure}[b]{0.485\textwidth}
        \centering
        \includegraphics{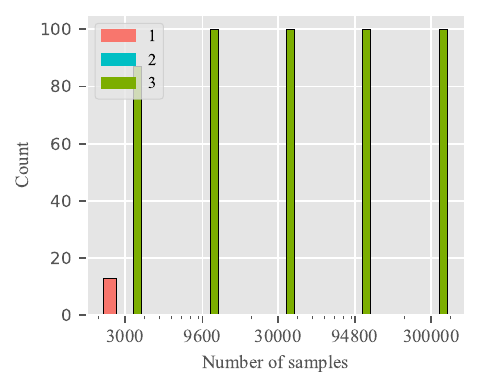}
        \caption{Gumbel distribution}
    \end{subfigure}\hfill
    \begin{subfigure}[b]{0.485\textwidth}
        \centering
        \includegraphics{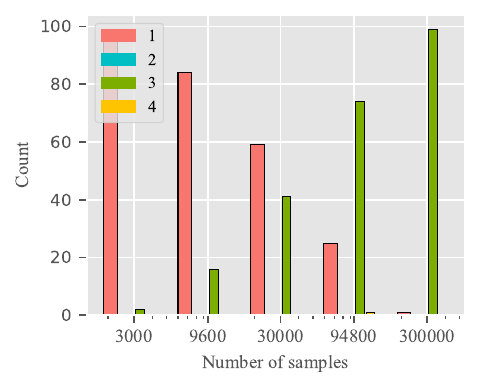}
        \caption{Logistic distribution}
    \end{subfigure}
    \caption{Histogram of $n$ in Algorithm 4.}
    \label{fig:hist_n_alg4}
\end{figure*}
\end{document}